\documentclass[letterpaper, 10 pt, conference]{ieeeconf}  

\IEEEoverridecommandlockouts                              

\usepackage{graphicx,subfigure}
\usepackage[font=small]{caption}
\usepackage{amssymb}
\usepackage{amsmath}
\usepackage{cite}
\usepackage{color}
\usepackage{comment}
\usepackage{enumerate}
\usepackage{listings}
\usepackage{url}
\usepackage{epstopdf}
\usepackage{array}
\usepackage{tikz}
\usepackage{amsmath}
\usepackage{algorithm}
\usepackage{varwidth}
\usepackage{algpseudocode}
\usepackage{multirow}
\usepackage{flexisym}
\usepackage[mathscr]{euscript}
\usepackage[english,ngerman]{babel}
\newtheorem{definition}{Definition}
\newtheorem{remark}{Remark}
\newtheorem{proposition}{Proposition}

\let\oldReturn\Return
\renewcommand{\Return}{\State\oldReturn}

\makeatletter
\def\BState{\State\hskip-\ALG@thistlm}
\makeatother

\title{\LARGE \bf
Probabilistically Safe Corridors \\ to Guide  Sampling-Based Motion Planning
}

\author{Jinwook Huh$^{1*}$, \"{O}m\"{u}r Arslan$^{2*}$, and
Daniel D. Lee$^{3}$
\thanks{$^*$ These authors contributed equally to this work.}
\thanks{$^{1}$ J. Huh is with the General Robotics, Automation, Sensing, and Perception (GRASP) Laboratory, University of Pennsylvania, Philadelphia, PA 19104. E-mail: jinwookh@seas.upenn.edu} \thanks{$^{2}$ \"{O}m\"{u}r Arslan is with the Autonomous Motion Department at Max Planck Institute for Intelligent Systems, T\"{u}bingen, Germany. E-mail: omur.arslan@tuebingen.mpg.de}
\thanks{$^{3}$ Daniel ~D. Lee is at Cornell Tech, New York, NY 10044. E-mail: ddl46@cornell.edu}
}%
\newcommand{\newtext}[1]{{\color{blue}#1}}
\newcommand{\rewrite}[1]{{\color{red}#1}}

\newcommand{\R}{\mathbb{R}} 
\newcommand{\N}{\mathbb{N}} 
\newcommand{\tr}[1]{{#1}^{\mathrm{T}}}
\newcommand{\vect}[1]{\mathrm{#1}} 
\newcommand{\diff}{\mathrm{d}} 
\newcommand{\norm}[1]{\left\|#1\right\|} 
\newcommand{\diag}{\mathrm{diag}} 
	\let\originalleft\left
	\let\originalright\right
	\renewcommand{\left}{\mathopen{}\mathclose\bgroup\originalleft}
	\renewcommand{\right}{\aftergroup\egroup\originalright}
\newcommand{\prl}[1]{\left(#1\right)} 
\newcommand{\crl}[1]{\left\{#1\right\}} 
\newcommand{\brl}[1]{\left[#1\right]} 

\newcommand{\confspace}{\mathcal{C}} 
\newcommand{\freespace}{\mathcal{F}} 
\newcommand{\collspace}{\mathcal{O}} 
\newcommand{\dimspace}{n} 

\newcommand{\gmpdf}{\mathcal{GM}} 
\newcommand{\gpdf}{\mathcal{N}} 
\newcommand{\mean}{\mu} 
\newcommand{\covmat}{\mathrm{\Sigma}} 
\newcommand{\weight}{\omega} 
\newcommand{\nummixture}{K} 
 
\newcommand{\point}  				{\vect{x}} 
\newcommand{\pdf}                   {p} 
\newcommand{\cdf}                   {F} 
\newcommand{\Lset}      			{\mathcal{L}} 
\newcommand{\supLset}      		{\Lset} 
\newcommand{\Lval}              	{\tau} 
\newcommand{\CR} 					{\mathcal{C}} 
\newcommand{\CL} 					{\kappa} 
\newcommand{\Lmap}               {L} 

\newcommand{\safespace}      {\mathcal{SC}} 

\newcommand{\mtrcproj}{\mathrm{\Pi}} 
\begin{document}

\maketitle
\thispagestyle{empty}
\pagestyle{empty}

\begin{abstract}

In this paper, we introduce a new probabilistically safe local steering primitive for sampling-based motion planning in complex high-dimensional configuration spaces.
Our local steering procedure is based on a new notion of a convex probabilistically safe corridor that is constructed around a configuration using tangent hyperplanes of confidence ellipsoids of Gaussian mixture models  learned from prior collision history.
Accordingly, we propose to expand a random motion planning graph towards a sample goal using its projection onto probabilistically safe corridors, which efficiently exploits the local geometry of configuration spaces for selecting proper steering direction and adapting  steering stepsize.
We observe that the proposed local steering procedure generates effective steering motion around difficult  regions of configuration spaces, such as narrow passages, while minimizing collision likelihood. We evaluate the proposed steering method with randomized motion planners in a number of planning scenarios, both in simulation and on a physical 7DoF robot arm, demonstrating the effectiveness of our safety guided local planner over the standard straight-line planner. 
\end{abstract}
\vspace{-0.15mm}

\section{Introduction}

Due to its simplicity and flexibility in handling a diverse set of configuration spaces without requiring an explicit representation, sampling-based motion planning is the mainstream approach to global motion planning for high-dimensional, highly nonlinear robotic systems, such as robot manipulators \cite{kavraki1996probabilistic,lavalle2001randomized,hsu1997path,karaman2011anytime}. 
However, the performance of such randomized motion planners strongly depends on the choice of distance measure, sampling method, and local steering; and is known to degrade significantly around complicated regions of configuration spaces, such as narrow passages~\mbox{\cite{hsu1998finding,lindemann_lavalle_ISER2005}}. 

This performance degrade is usually considered as an issue of sampling, because uniform sampling  has a Voronoi bias towards yet unexplored larger regions of configuration spaces; and accordingly many heuristic rejection sampling approaches and retraction methods are suggested to mitigate this issue, but retraction methods often require a distance-to-collision measure \cite{saha2005finding,zhang2008efficient}. 
On the contrary, assuming that this performance decay is due to the lack of effective local steering, in \cite{arslan_pacelli_kod_IROS2017}  a geometric local steering policy that can ``feel'' the local geometry of configuration spaces  is proposed for efficient planning around narrow passages; however, its computation also requires a distance-to-collision measure.
Since the exact computation of distance-to-collision  in complex high-dimensional configuration spaces is hard \cite{denny2013adapting}, Gaussian mixture learning \cite{huh2016learning} and locally weighted regression \cite{burns_brock_ICRA2005} are applied to construct approximate probabilistic models of collision and collision-free subspaces of  configuration spaces for fast collision checking and biased sampling over free space and difficult regions of configuration spaces.
In particular, simultaneous modeling of collision and free subspaces is shown to be critical for local planning around narrow passages \cite{denny_amato_WAFR2013}.
In this paper, by combining the strengths of  \cite{arslan_pacelli_kod_IROS2017} and \cite{huh2016learning}, we introduce a new notion of \emph{probabilistically safe corridors} for probabilistically safe guided local steering for sampling-based planning without requiring an explicit computation of distance-to-collision.

\begin{figure}[t]
\centering
\vspace{1mm}
\begin{tabular}{@{}c@{\hspace{2mm}}c@{}}
\includegraphics[width=0.235\textwidth]{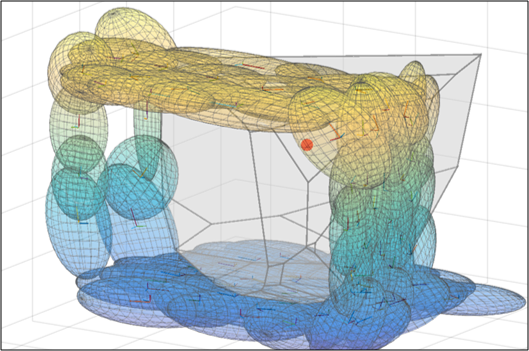}
&
\includegraphics[width=0.235\textwidth]{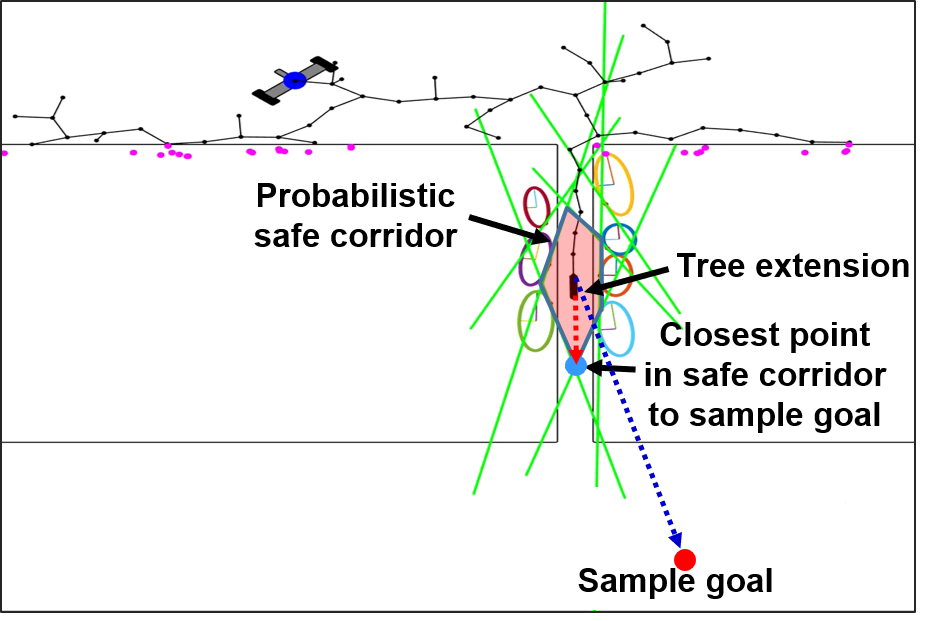}
\end{tabular}
\vspace{-2mm}
\caption{(left) Probabilistically safe corridor in 3D space constructed around a sample configuration (red) by using tangent hyperplanes (gray)  of confidence ellipsoids of a learned Gaussian mixture model of configuration space obstacles. (right) Local steering via probabilistically safe corridor in 2D space: An RRT is extended along the safe direction (red dotted line) towards the projection of a sample goal (red) onto the associated probabilistically safe corridor (red polygon),  instead of the standard straight-line extension (blue dotted line) towards the sample goal.}
\label{grd_fig:polytope}
\vspace{-4mm}
\end{figure}

More precisely, we construct a probabilistically safe corridor around a configuration using tangent hyperplanes of  confidence regions of learned Gaussian mixtures that separate the input configuration from the confidence ellipsoids, as illustrated in Fig. \ref{grd_fig:polytope}\,(left).
Accordingly, we propose a probabilistically safe local steering primitive towards a sample goal configuration via its projection onto the probabilistically safe corridor, as shown in Fig. \ref{grd_fig:polytope}\,(right). 
Since the proposed steering method exploits the local geometry of configuration spaces via learned Gaussian mixture models (GMMs) and generates steering motion within probabilistically safe corridors, in our numerical simulation and experiments, we observe that it yields a better exploration of configuration spaces while minimizing collision~likelihood.

In summary, the main contributions of the paper include: 
\begin{enumerate}[i)]
\item a novel geometric approximation of configuration space obstacles by confidence ellipsoids of learned GMMs,
\item  a new construction of probabilistically safe corridors using tangent hyperplanes of confidence ellipsoids, 
\item an effective probabilistically safe local steering primitive that can minimize collision likelihood.
\end{enumerate}
Using numerical simulations and real experiments, we demonstrate that the proposed probabilistically safe local steering approach can dramatically improve the performance of randomized motion planners around narrow passages and significantly outperforms the straight-line local planner in high dimensional configuration spaces by decreasing the number of collisions.

\section{Related Work}

Sampling-based planning approaches suffer from heavy computational time in complex environments since they typically require a considerable number of sample configurations and their collision checks. Therefore, several biased sampling methods \cite{boor1999gaussian,hsu1998finding} and rejection sampling methods \cite{shkolnik2011sample,shkolnik2009reachability,yershova2005dynamic} are proposed to reduce the number of sample nodes and so to improve computational efficiency. However, these approaches have many heuristic parameters and require explicit configuration space information, such as visibility or collision boundaries, which usually limits their application to  low dimensional settings. 
Another alternative approach to increase the computation efficiency is to reduce the number of collision checks, using either lazy collision checking \cite{hwang2015lazy,sanchez2002delaying,bohlin2000path}  or fast probabilistic collision checks  \cite{huh2016learning,huh2017adaptive,panfast,aoude2013probabilistically}. 
Exact safety certificates are also utilized for minimizing the computational cost of collision checks \cite{bialkowski_otte_karaman_frazzoli_IJRR2016}.
However, these methods are still not able to address the narrow passage problem of sampling-based motion planning.

In order to resolve the narrow passage problem, Zhang and Manocha present a steering approach that retracts sample configurations to become more likely to be connected to nearby nodes \cite{zhang2008efficient}. However, it requires a significant number of iterations to find a new collision-free configuration that is around the collision boundary, and also requires an appropriate distance-to-collision measure. In practice, since the exact distance-to-collision measurement in high dimensional configuration spaces is very hard, its applicability is also limited to low dimensional motion planning problems.
Moreover, workspace topology is utilized in biasing configuration space exploration  for planning around difficult regions \cite{plaku2010motion, denny2016dynamic}, but the topology of high-dimensional configuration space (e.g., robot manipulators) is significantly different and more complex than the corresponding workspace topology.    

Local safe corridors \cite{wein_berg_halperin_IJRR2008, geraerts_ICRA2010, chen2016online, liu2017planning} recently find significant applications in collision-free motion planning by using sequential composition of simple local planners \cite{conner_rizzi_choset_IROS2003}. 
Such safe corridors  are usually constructed based on a convex decomposition of the environment, which requires an explicit representation of the environment.
In \cite{arslan_pacelli_kod_IROS2017},  a sensory steering algorithm is proposed for sampling-based motion planning that increases the connectivity of randomized motion planning graphs, especially around narrow passages, by exploiting local geometry of configuration spaces via convex local safe corridors.
This construction is further extended to integrate local system dynamics and local workspace geometry in kinodynamic motion planning \cite{pacelli_arslan_koditschek_ICRA2018}.
However,  the original construction of sensory steering requires an explicit representation of configuration space obstacles or an explicit distance-to-collision metric, and so its direct application to high dimensional motion planning is limited. In this paper, we enhance this sensory steering algorithm to adapt it to high dimensional settings, such as robotic manipulation, by defining probabilistically safe corridors that are constructed using a learned approximate probabilistic model of a configuration space.

\section{Safety-Guided RRT \\ via Probabilistically Safe Corridors}


In this section, we first present a brief overview of how learning of Gaussian mixtures\footnote{\label{ft.GMM}Although other probabilistic (mixture) models can be used for approximating $\freespace$ and $\collspace$, we find it convenient to use Gaussian mixtures since their confidence regions can be accurately and efficiently approximated using confidence regions of individual Gaussians which have an ellipsoidal form.} can be used for approximate probabilistic modeling of configuration spaces, and then introduce a new notion of a probabilistically safe corridor around a configuration that identifies a safe neighborhood of  the configuration with minimal collision risk.
Accordingly, we propose a practical extension\footnote{Safety guided steering via probabilistically safe corridors can be integrated with any (sampling-based) motion planning algorithm (e.g., probabilistic roadmaps--PRMs) as a local steering primitive, especially for uncertainty-aware belief-space planning, which we plan to explore in a future paper.} of the standard RRT planner, called Safety-Guided RRT (SG-RRT), where tree extension is guided to ensure safety constraints defined by  probabilistically safe corridors.


\vspace{-1mm}
\subsection{Gaussian Mixture Modeling of Configuration Spaces}
\label{sec.ProbabilisticModel}

Let $\confspace$ denote the configuration space of a robotic system embedded in an $\dimspace$-dimensional Euclidean space $\R^n$, and denote by $\freespace \subset \confspace$ and $\collspace\subset \confspace$, respectively, the free subspace and the collision subspace (i.e., obstacles) of the configuration space $\confspace$, which, by definition, satisfy $\freespace = \confspace \setminus \collspace$.
In general, an explicit representation of the free space $\freespace$ or the collision space $\collspace$ in terms of simple geometric  shapes is known to be very hard to obtain, especially for high-dimensional complex systems such as robotic manipulators.
Hence, as in \cite{huh2016learning}, we consider approximate probabilistic representations of the free space $\freespace$ and the collision space $\collspace$ in terms of Gaussian mixtures models\textsuperscript{\ref{ft.GMM}}, respectively, denoted by $\gmpdf\prl{\boldsymbol{\mean}_{\freespace}, \boldsymbol{\covmat}_{\freespace}, \boldsymbol{\weight}_{\freespace}}$ and $\gmpdf\prl{\boldsymbol{\mean}_{\collspace}, \boldsymbol{\covmat}_{\collspace}, \boldsymbol{\weight}_{\collspace}}$, that are constructed using collision and collision-free sample configurations as described below.
Here, a Gaussian mixture distribution $\gmpdf\prl{\boldsymbol{\mean}, \boldsymbol{\covmat}, \boldsymbol{\weight}}$, consisting of $\nummixture \in \N$ mixture components, is parametrized by  a list of mixture means $\boldsymbol{\mean} := (\mean_{1}, \mean_{2}, \ldots, \mean_{\nummixture}) \in \prl{\R^{\dimspace}}^{\nummixture}$,  a list of positive-definite covariance matrices  $\boldsymbol{\covmat} := (\covmat_{1}, \covmat_{2}, \ldots, \covmat_{\nummixture}) \in \prl{\R^{\dimspace\times \dimspace}}^{\nummixture}$ and  a list of normalized mixture weights $\boldsymbol{\weight}:=(\weight_{1}, \weight_{2}, \ldots, \weight_{\nummixture}) \in \prl{\R_{\geq 0}}^{\nummixture}$, satisfying $\sum_{k=1}^{\nummixture}\weight_k = 1$,  and its value at a point $\vect{x} \in \R^{\dimspace}$ is given by%
 \vspace{-2mm}
 \begin{align}
 \gmpdf\prl{\vect{x}; \boldsymbol{\mean}, \boldsymbol{\covmat}, \boldsymbol{\weight}} := \sum\limits_{k=1}^{\nummixture} \weight_i \gpdf\prl{\vect{x}; \mean_k, \covmat_k},
 \end{align}
 where  $\gpdf\prl{\vect{x}; \mean, \covmat}$ is the multivariate Gaussian distribution with mean $\mean$ and covariance matrix $\covmat$,
 \begin{align}
 \!\gpdf\prl{\vect{x}; \mean, \covmat} \!:=\! \frac{1}{\det\prl{2\pi\covmat}^{\frac{1}{2}}}\exp\prl{\!\!-\frac{1}{2}\tr{\prl{\vect{x}\!-\!\mean}}\covmat^{-1}\prl{\vect{x}\!-\!\mean}\!\!}.\!\!\!
 \end{align}
Note that the numbers of mixtures, $\nummixture_{\freespace}$ and $\nummixture_{\collspace}$, used for modeling the free space $\freespace$ and the collision space $\collspace$ can be different, especially the Meanshift  clustering algorithm used in this paper automatically determines the number of mixture components using sample configurations based on a geometric bandwidth parameter as described below. 
It is also important to highlight that one can simply use $\gmpdf\prl{\vect{x}, \boldsymbol{\mean}_{\freespace}, \boldsymbol{\covmat}_{\freespace}, \boldsymbol{\weight}_{\freespace}}$  and $\gmpdf\prl{\vect{x}, \boldsymbol{\mean}_{\collspace}, \boldsymbol{\covmat}_{\collspace}, \boldsymbol{\weight}_{\collspace}}$ to estimate how likely a configuration is   in collision, which is leveraged in \cite{huh2016learning} for fast collision checking and biased sampling.  
In addition to such demonstrated  potential improvements, we shall show below that confidence regions of these Gaussian mixture models can be utilized for  understanding the local geometry of the configuration space $\confspace$   and for increasing the quality of the local steering heuristic (which is the Euclidean distance in our case) to  better approximate the true geodesic  (cost-to-go)  metric  of the configuration space $\confspace$.

\subsubsection{Learning Gaussian Mixtures}

\begin{figure}[t]
\vspace{1.5mm}
\centering
\begin{tabular}{@{\hspace{0.5mm}}c@{\hspace{1mm}}c@{}}
\includegraphics[width=0.265\textwidth]{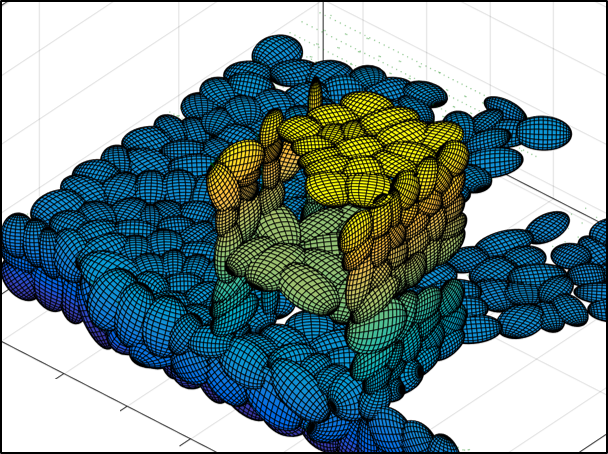}
\includegraphics[width=0.200\textwidth]{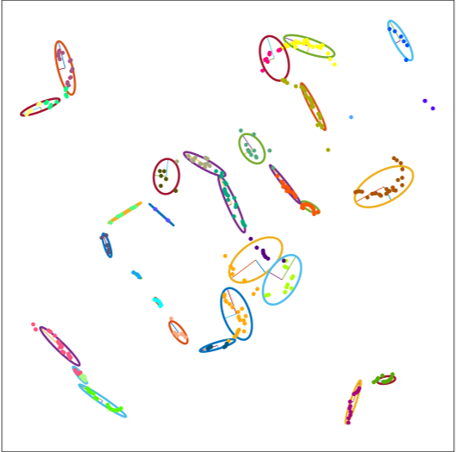}
\end{tabular}
\vspace{-2mm}
\caption{Examples of learned Gaussian mixture models. Ellipsoids show the confidence regions associated with the confidence level of $\kappa = 0.9$. (left) Gaussian mixtures in the 3D workspace shown in Fig. \ref{grd_7d_link_result}, (right) Gaussian mixtures in the configuration space of a 2DoF planar manipulator.}
\label{grd_fig_modeling_result1}
\vspace{-3.5mm}
\end{figure}


One can use a number of Expectation-Maximization (EM) variant methods for  Gaussian mixture learning  for modeling the free space $\freespace$ and  the collision space $\collspace$ using  collision and collision-free sample configurations in an offline or online manner, as in our previous work \cite{huh2016learning}. In this paper, we apply the Meanshift clustering method \cite{cheng_PAMI1995} with a Gaussian kernel for learning Gaussian mixtures using collision information of sample configurations obtained during previous attempts of a randomized motion planner, which is a convenient way of learning from past experiences and exploiting the collision history. In addition, this approach resolves the problem that general mixture modeling approaches have no explicit way of determining the required number of mixtures, because the Meanshift clustering requires a kernel bandwidth $B$ instead of the number of clusters $K$.
The kernel bandwidth $B$ can be set based on the desired level of spatial resolution.
With the bandwidth $B$, we initialize the clusters and then perform a single step EM update to estimate cluster statistics. We set the membership weight value  as $z^i_k = 1$ if the $i$th point in $N$ samples is included in the $k$th cluster, and $z^i_k = 0$ otherwise. Then, the cluster statistics (mass $m_k$, mean $\mean_k$, covariance matrix $\covmat_k$, and weight $\weight_k$) for the $k$th cluster are given by
\begin{align*}
m_k =& \sum_{i = 1}^N z^i_k, ~~\mean_k = \frac{1}{m_k}\sum_{i=1}^N z^i_k \vect{x}_i, ~~ \weight_k =\frac{m_k}{\sum_{j=1}^K m_j}~,\\
\covmat_k =& \frac{1}{m_k}\sum_{i = 1}^N z^i_k (\vect{x}_i \!-\! \mean_k)\tr{(\vect{x}_i \!-\! \mean_k)}\!, \,
\textrm{for} ~ k \in \{ 1,\cdots,K\}.
\end{align*}


In Fig. \ref{grd_fig_modeling_result1},  we present some examples of constructed probabilistic models of different configuration space and workspace by the suggested approach. 
Fig. \ref{grd_fig_modeling_result1}\,(left) shows a probabilistic model to define the collision space from 3D point clouds obtained by a depth sensor. 
Fig. \ref{grd_fig_modeling_result1}\,(right) shows the generated probabilistic models using collision information of samples in the configuration space of a 2DoF planar manipulator. Such  probabilistic representations of configuration spaces can be utilized for collision likelihood estimation, as a computationally efficient alternative to the exact distance-to-collision measurement \cite{huh2016learning}.  


\subsubsection{Confidence Regions of Gaussian Mixtures} 
\label{sec.GM_ConfidenceRegion}

While a Gaussian mixture model $\gmpdf\prl{\boldsymbol{\mean}_{\freespace}, \boldsymbol{\covmat}_{\freespace}, \boldsymbol{\weight}_{\freespace} }$ of the free space $\freespace$ can be used to bias sampling over the free space, in addition to its use in fast collision checking \cite{huh2016learning}, we propose  a new novel  use of confidence regions of a Gaussian mixture model $\gmpdf\prl{\boldsymbol{\mean}_{\collspace}, \boldsymbol{\covmat}_{\collspace}, \boldsymbol{\weight}_{\collspace}}$ of the collision space $\collspace$  for understanding the local geometry of the configuration space $\confspace$, which is the main contribution of the present paper.

\begin{definition}
The \emph{confidence region} $\CR_{\pdf}\prl{\CL}$ of a continuous probability distribution $\pdf:\R^\dimspace \rightarrow \R_{\geq 0}$ associated with a \emph{confidence level}  $\CL \in \brl{0,1}$ is defined to be the super level set $\supLset_{\pdf}\prl{\Lval} := \crl{\point \in \R^{\dimspace} \Big| \, \pdf\prl{\point} \geq \Lval}$ of $p$,  for some $\Lval \in \R_{\geq 0}$, over which the cumulative mass distribution of $\pdf$ is $\CL$, i.e, 
\begin{align}
\CR_{\pdf}\prl{\CL} \:= \supLset_{\pdf}\prl{\Lval}  \quad \text{such that} \quad  \int_{\supLset_{\pdf}\prl{\Lval}} \pdf\prl{\point} \diff \point = \CL \,.
\end{align}
Hence, it is  convenient to have  $\Lmap_{\pdf}\prl{\CL}$ denote the level function of  $\pdf$ that returns the corresponding level of $\pdf$ defining the confidence region $\CR_{\pdf}\prl{\CL}$, i.e.,%
\begin{align}
\CR_{\pdf}\prl{\CL} = \supLset_{\pdf}\prl{\Lmap_{\pdf}\prl{\CL}}.
\end{align}
\end{definition}

\smallskip

Although confidence regions of an arbitrary probability distribution cannot be expressed explicitly in terms of simple geometric shapes and so are needed to be computed numerically \cite{hyndman_AS1996},
confidence regions of Gaussian distributions have an analytical ellipsoidal form. 

\begin{figure}[t]
\vspace{2mm}
\centering
\begin{tabular}{@{}c@{}c@{}}
\includegraphics[width=0.24\textwidth]{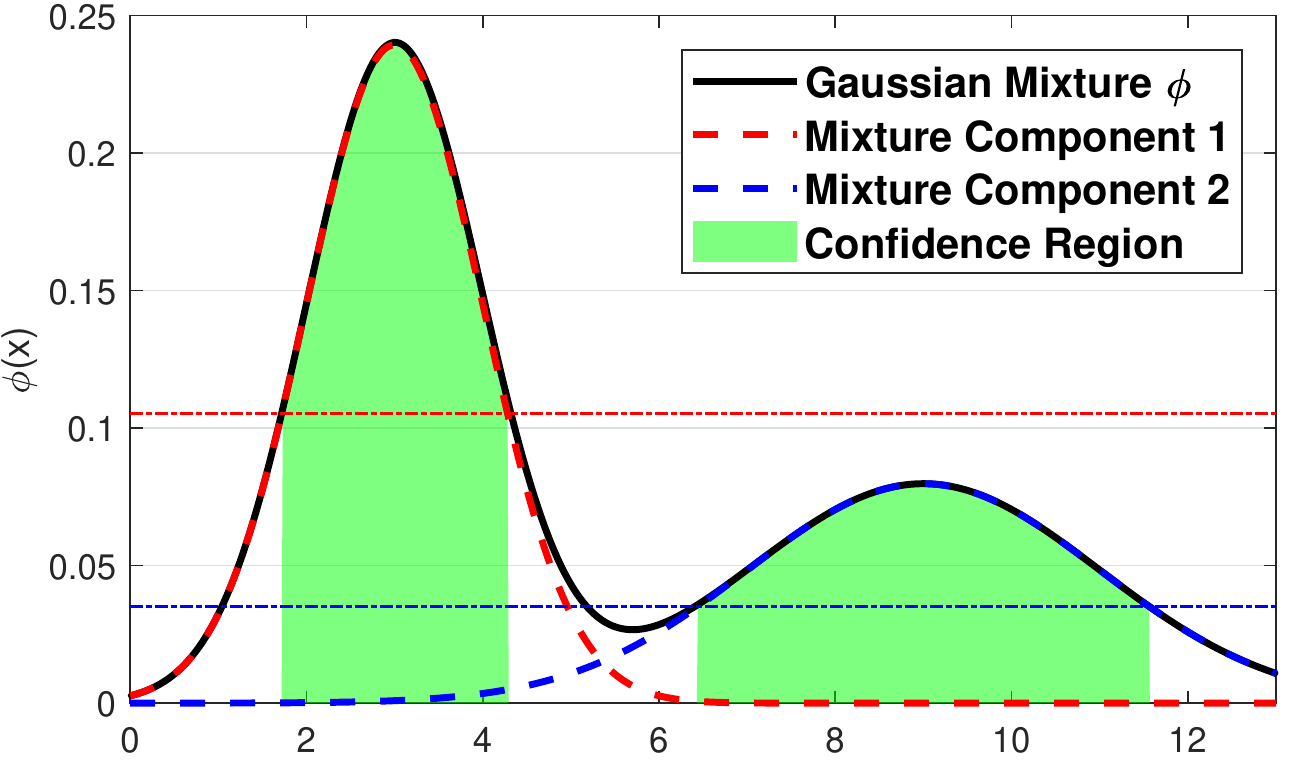} &  \includegraphics[width=0.24\textwidth]{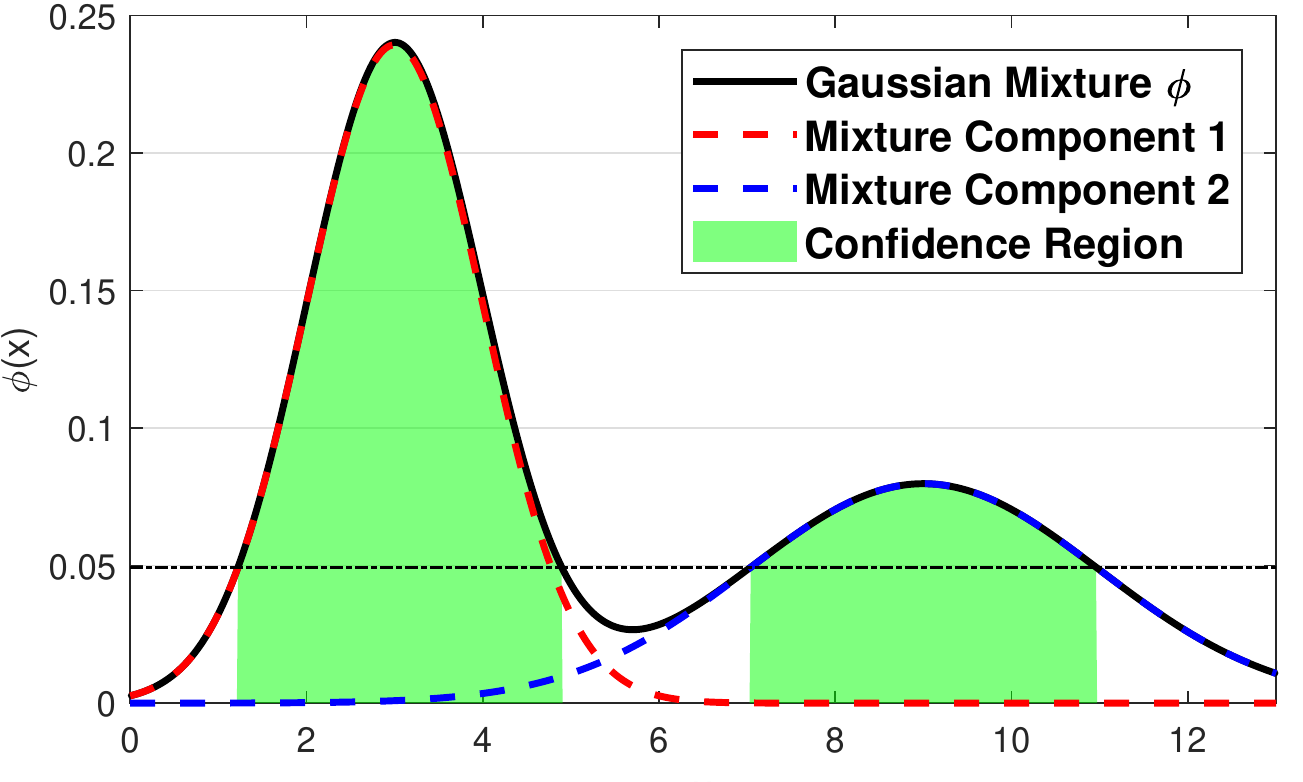} \\[-2.5mm]
\hspace{2.5mm}\scalebox{0.7}{(a)} & \hspace{2.5mm}\scalebox{0.7}{(b)} \\[0mm]
\hspace{2.5mm}\includegraphics[width=0.22\textwidth]{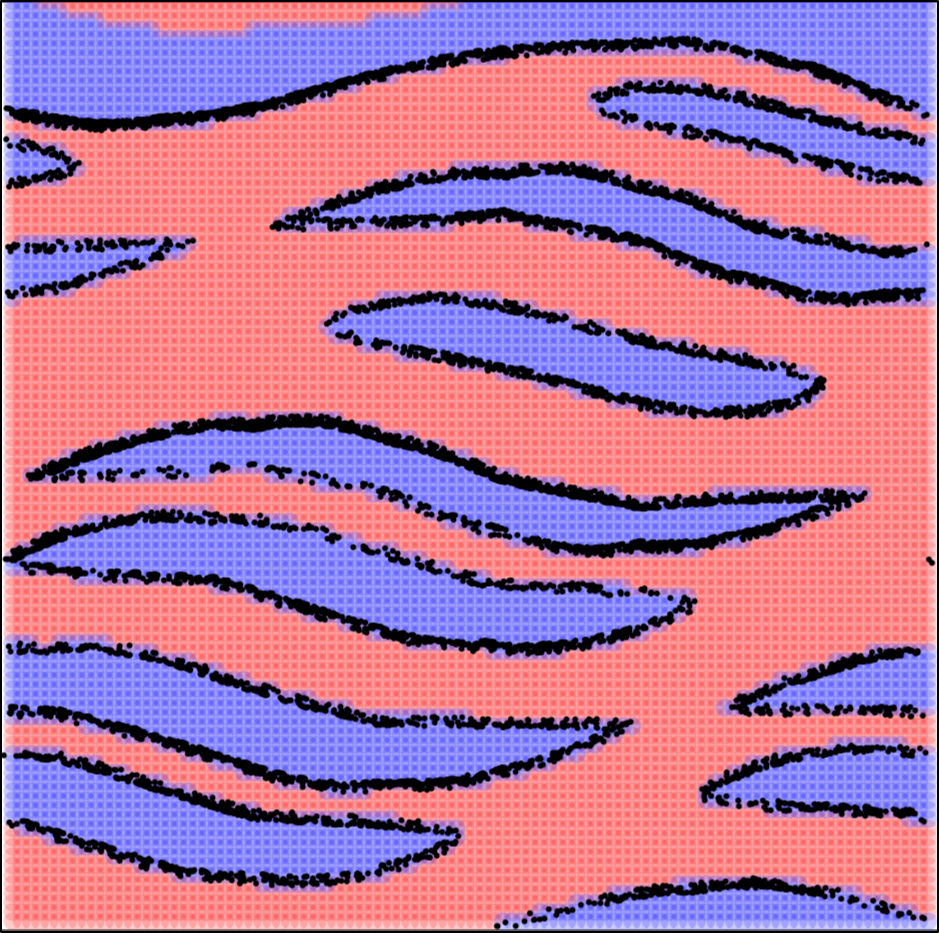} & \hspace{2.5mm}\includegraphics[width=0.22\textwidth]{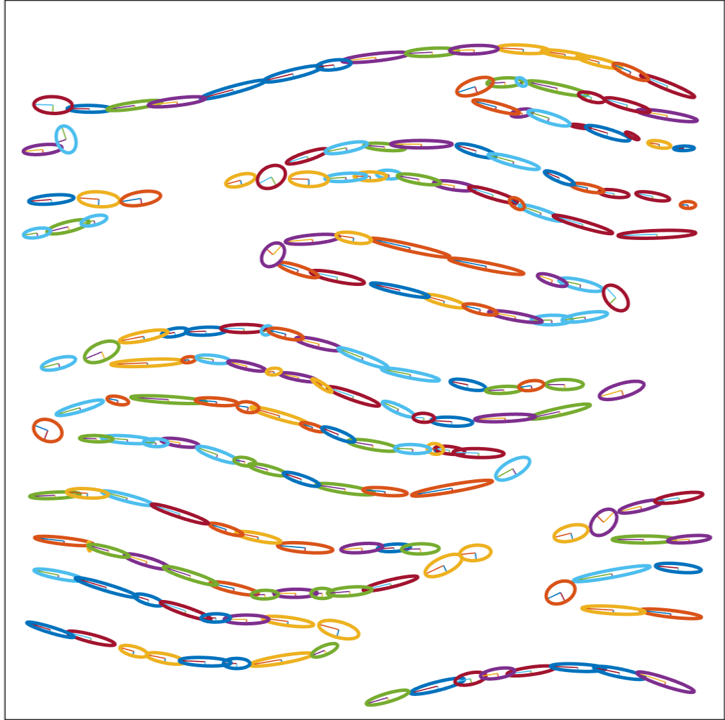} \\[-1mm]
\hspace{2.5mm}\scalebox{0.7}{(c)} & \hspace{2.5mm}\scalebox{0.7}{(d)}
\end{tabular}
\vspace{-3mm}
\caption{GMM confidence regions. (a)  Super level sets  of individual Gaussians at confidence level $\CL_k = \CL$. (b) Super level sets of Gaussians at the confidence levels corresponding to a shared probability level. (c) An example configuration space (collisions are in blue and free space is in red)  and (d) the associated confidence ellipsoids of learned GMM distributions from collision samples (black in (c)). }
\label{gmm_confidence_region}
\vspace{-3mm}
\end{figure}


\begin{remark}
For any confidence level $\CL \in \brl{0,1}$, the ellipsoidal  confidence region $\CR_{\gpdf\prl{\mean. \covmat}}\prl{\CL}$ and the level function $\Lmap_{\gpdf\prl{\mean,\covmat}}\prl{\CL}$ of the Gaussian distribution $\gpdf\prl{\point; \mean, \covmat}$ are, respectively, given by  
\begin{align}
\label{eq.gCR}
\!\!\CR_{\gpdf\prl{\mean, \covmat}}\prl{\CL} & =\! \crl{\point \!\in \R^{\dimspace} \bigg| \tr{\prl{\point \!-\! \mean}} \covmat^{-1} \prl{\point \!-\! \mean} \leq \cdf_{\chi^2_{\dimspace}}^{-1}\prl{\CL}\!},\!\!\! 
\\
\!\!\Lmap_{\gpdf\prl{\mean,\covmat}}\prl{\CL} & = \frac{1}{\det\prl{2\pi\covmat}^{\frac{1}{2}}} \exp\prl{- \frac{1}{2}\cdf_{\chi^2_{\dimspace}}^{-1}\prl{\CL}\!},
\end{align}
where $\cdf_{\chi^2_{\dimspace}}: \R_{\geq0} \rightarrow \brl{0,1}$ denotes the cumulative probability distribution of $\chi^2_{\dimspace}$ distribution with $\dimspace$ degrees of freedom.
Hence, for any $\Lval \in \R_{\geq 0}$, the confidence level $\CL$ of the super level set $\supLset_{\gpdf\prl{\mean,\covmat}}\prl{\Lval}$ of the Gaussian distribution $\gpdf\prl{\mean, \covmat}$ is explicitly given by
\begin{align}
\CL = \Lmap_{\gpdf\prl{\mean,\covmat}}^{-1}\prl{\Lval} = \cdf_{\chi^2_{\dimspace}}\prl{-\log\prl{\Lval^2 \det \prl{2\pi \covmat}\!}\!}.
\end{align} 
\end{remark}

\medskip 

Accordingly, since it lacks an exact closed-form expression, we suggest approximating the confidence region of a Gaussian mixture distribution $\gmpdf\prl{\boldsymbol{\mean}, \boldsymbol{\covmat}, \boldsymbol{\weight}}$ associated with a confidence level $\CL \in \brl{0,1}$ as a union of ellipsoidal confidence regions of individual Gaussians,  associated with confidence levels $\boldsymbol{\CL}:=\prl{\CL_1, \CL_2, \ldots, \CL_\nummixture}$ that satisfy $\sum_{k=1}^{\nummixture}\weight_k\CL_k = \CL$ , as
\begin{align}
& \overline{\CR}_{\gmpdf\prl{\boldsymbol{\mean}, \boldsymbol{\covmat}, \boldsymbol{\weight}}}\prl{\boldsymbol{\CL}} := \bigcup\nolimits_{k=1}^{\nummixture} \CR_{\gpdf\prl{\mean_k, \covmat_k}}\prl{\CL_k}, \\
& =  \bigcup_{k=1}^{\nummixture}  \crl{\point \in \R^{\dimspace} \Big| \tr{\prl{\point - \mean_k}} \covmat_k^{-1} \prl{\point - \mean_k} \leq \cdf_{\chi^2_{\dimspace}}^{-1}\prl{\CL_k}\!},\!\!\! .
\end{align} 
Observe that, by construction, we have
\begin{equation}
\int\nolimits_{\overline{\CR}_{\gmpdf\prl{\boldsymbol{\mean}, \boldsymbol{\covmat}, \boldsymbol{\weight}}}\prl{\boldsymbol{\CL}}} \gmpdf\prl{\point; \boldsymbol{\mean}, \boldsymbol{\covmat}, \boldsymbol{\weight}} \diff \point \, \geq \, \CL\,.
\end{equation}
A standard choice of the confidence levels of individual Gaussians is $\CL_k = \CL$ for all $k$ as shown in Fig. \ref{gmm_confidence_region}\,(a); however, this usually yields a poor approximation of the actual confidence region of the mixture model because less accurate Gaussians with high variances become more influential in determining the confidence region.
A more accurate analytical choice for  the individual confidence levels is $\CL_k = \Lmap_{\gpdf\prl{\mean_k,\covmat_k}}^{-1}\prl{\frac{\Lval}{\weight_k}} $ based on a shared probability level  $\Lval = \sum_{k=1}^{\nummixture} \weight_k^2 \Lmap_{\gpdf\prl{\mean_k,\covmat_k}}\prl{\CL}$ \cite{arslan_huh_lee_2018}.
Alternatively, in this paper, we use  an iterative search algorithm to find a more accurate shared probability level  $\Lval$  as described in \cite{arslan_huh_lee_2018}  and set  $\CL_k = \Lmap_{\gpdf\prl{\mean_k,\covmat_k}}^{-1}\prl{\frac{\Lval}{\weight_k}} $ for all $k$, as shown in Fig. \ref{gmm_confidence_region}\,(b). With this approach, we obtain confidence regions of  Gaussian mixture models that approximately represents  configuration space obstacles, as illustrated in  Fig. \ref{gmm_confidence_region}\,(c)-(d).

\subsection{Probabilistically Safe Corridors}

Suppose $\gmpdf\prl{\boldsymbol{\mean}_{\collspace}, \boldsymbol{\covmat}_{\collspace},\boldsymbol{\weight}_{\collspace}}$ be a Gaussian mixture model constructed as described above for modeling the collision subspace $\collspace$ of a configuration space in $\R^{\dimspace}$  and let $\overline{\CR}_{\gmpdf\prl{\boldsymbol{\mean}_{\collspace}, \boldsymbol{\covmat}_{\collspace}, \boldsymbol{\weight}_{\collspace}}}\prl{\boldsymbol{\CL}_{\collspace}}$ be the corresponding approximate confidence region associated with a desired confidence level $\CL = \sum_{k=1}^{\nummixture_{\collspace}} \weight_{\collspace_k}\CL_{\collspace_k}$.
Accordingly, we define the \emph{probabilistically safe corridor} around a configuration $\vect{p} \in \R^{\dimspace}$ to be%
{\small
\begin{align}
\safespace_{\collspace}\prl{\vect{p}} &\!\!:=\!\! \crl{\!\vect{x}\bigg|
\scalebox{0.89}{$\tfrac{\!\tr{\prl{\vect{p}-\mean_{\collspace_k}}\!} \covmat_{\collspace_k}^{-1} \!\prl{\vect{x} - \mean_{\collspace_k}}\!}{\!\big\|\covmat_{\collspace_k}^{-\frac{1}{2}}\!\prl{\vect{p}-\mean_{\collspace_k}}\big\|^2\!}$}\!\geq\!  \min \!\left(\! \scalebox{0.89}{$\tfrac{\sqrt{\cdf_{\chi^2_{\dimspace}}^{-1}\prl{\CL_{\collspace_k}}}}{\!\big\|\covmat_{\collspace_k}^{-\frac{1}{2}}\!\prl{\vect{p}-\mean_{\collspace_k}}\big\|\!}$}, \!1 \!\!-\!\! \epsilon\!\!\right)\!,  \forall k  \!}\!,
\\
&\hspace{-11mm}\!=\! \crl{\!\vect{x} \!\in\! \R^{\dimspace} \! \bigg| \tfrac{\tr{\prl{\mean_{\collspace_k}\!\!- \vect{p}}\!} \covmat_{\collspace_k}^{-1}\! \prl{\vect{x} - \vect{p}}}{\Big\|\covmat_{\collspace_k}^{-\frac{1}{2}}\!\prl{\mean_{\collspace_k}\! \!- \vect{p}}\Big\|^2} \!\leq\!  \max \!\left( \!\!1\!-\! \tfrac{\sqrt{\cdf_{\chi^2_{\dimspace}}^{-1}\prl{\CL_{\collspace_k}}}}{\Big\|\covmat_{\collspace_k}^{-\frac{1}{2}}\!\prl{\mean_{\collspace_k}\! \!-\! \vect{p}}\Big\|}, \epsilon\!\!\right)\!,  \forall k \!}\!, \label{eq.ProbabilisticallySafeCorridor}
\end{align}
}%
%
%
which is constructed using tangent hyperplanes of confidence ellipsoids of Gaussians and is a closed convex polytope, as depicted Fig. \ref{grd_fig_concept_guidance}.
Here, $\epsilon \in \R$ is a scalar safety tolerance parameter, and $\norm{.}$ denotes the standard Euclidean norm, and for any positive-definite covariance matrix $\covmat \in \R^{\dimspace\times \dimspace}$, a positive-definite choice of $\covmat^{-\frac{1}{2}}$ is  $\covmat^{-\frac{1}{2}} = \vect{V} \prl{\diag\prl{\frac{1}{\sqrt{\sigma_1}}, \frac{1}{\sqrt{\sigma_2}}, \ldots, \frac{1}{\sqrt{\sigma_\dimspace}}}} \tr{\vect{V}}$ where $\covmat = \vect{V}\,\diag\prl{\sigma_1, \sigma_2, \ldots, \sigma_{\dimspace}} \tr{\vect{V}}$ is the singular-value decomposition of $\covmat$.
It is also useful to observe from (\ref{eq.gCR}) that $\cdf_{\chi^2_{\dimspace}}^{-1}\prl{\CL_{\collspace_k}} = \Big\|\covmat_{\collspace_k}^{-\frac{1}{2}}\!\prl{\mean_{\collspace_k}\! \!-\! \vect{p}}\Big\|^2$ for any confidence region boundary point  $\vect{p} \in \partial\CR_{\gpdf\prl{\mean_{\collspace_k}, \covmat_{\collspace_k}}}\prl{\CL_{\collspace_k}}$.
Hence, the safety constraints encoded by $\safespace_{\collspace}$ are relaxed with increasing $\epsilon$.

\begin{proposition}\label{prop.NonemptySafeSpace}
For $\epsilon \geq 0$, the probabilistically safe corridor $\safespace_{\collspace}\prl{\vect{p}}$ of a configuration $\vect{p} \in \R^{\dimspace}$ is a  nonempty convex neighborhood of $\vect{p}$; and for $\epsilon > 0$, $\safespace_{\collspace}\prl{\vect{p}}$ strictly contains $\vect{p}$ in its interior $\mathring{\safespace}_{\collspace}\prl{\vect{p}}$, i.e., for any  $\vect{p} \in \R^{\dimspace}$ 
\begin{align}
\vect{p} \in \safespace_{\collspace}\prl{\vect{p}} \quad \forall \epsilon \geq 0, \text{ and } \vect{p} \in \mathring{\safespace}_{\collspace}\prl{\vect{p}} \quad \forall \epsilon > 0.
\end{align}
\end{proposition}
\medskip
\begin{proof}
By definition (\ref{eq.ProbabilisticallySafeCorridor}), the probabilistically safe corridor $\safespace_{\collspace}\prl{\vect{p}}$ is constructed as an intersection of half-spaces and so is a convex polytope. 
Moreover, for any $\epsilon \geq 0$ (resp. $\epsilon > 0$), these half-spaces are guaranteed to  contain $\vect{p}$ (resp. strictly in their interiors). 
Thus, the result follows.
\end{proof}


\begin{proposition} \label{prop.RealSafeSpace}
For $\epsilon \leq 0$, the probabilistically safe corridor $\safespace_{\collspace}\prl{\vect{p}}$ of a probabilistically safe state $\vect{p} \in \R^{\dimspace} \setminus  \overline{\CR}_{\gmpdf\prl{\boldsymbol{\mean}_{\collspace}, \boldsymbol{\covmat}_{\collspace}, \boldsymbol{\weight}_{\collspace}}}\prl{\boldsymbol{\CL}_{\collspace}}$ contains $\vect{p}$ in its interior $\mathring{\safespace}_{\collspace}\prl{\vect{p}}$  and is  also  probabilistically safe, i.e.,%
\begin{align}
&\vect{p} \in \R^\dimspace\setminus \overline{\CR}_{\gmpdf\prl{\boldsymbol{\mean}_{\collspace}, \boldsymbol{\covmat}_{\collspace}, \boldsymbol{\weight}_{\collspace}}}\prl{\boldsymbol{\CL}_{\collspace}} \nonumber \\
&\hspace{8mm}\Longrightarrow
\vect{p} \in \mathring{\safespace}_{\collspace}\prl{\vect{p}} \subset \R^{\dimspace} \setminus \overline{\CR}_{\gmpdf\prl{\boldsymbol{\mean}_{\collspace}, \boldsymbol{\covmat}_{\collspace}, \boldsymbol{\weight}_{\collspace}}}\prl{\boldsymbol{\CL}_{\collspace}}. \!\!\!
\end{align}
\end{proposition}
\medskip
\begin{proof}
For any $\vect{p} \in \R^\dimspace\setminus \overline{\CR}_{\gmpdf\prl{\boldsymbol{\mean}_{\collspace}, \boldsymbol{\covmat}_{\collspace}, \boldsymbol{\weight}_{\collspace}}}\prl{\boldsymbol{\CL}_{\collspace}}$, we have from (\ref{eq.gCR}) that $\scalebox{0.9}{$\tfrac{\sqrt{\cdf_{\chi^2_{\dimspace}}^{-1}\prl{\CL_{\collspace_k}}}}{\!\big\|\covmat_{\collspace_k}^{-\frac{1}{2}}\!\prl{\vect{p}-\mean_{\collspace_k}}\big\|\!}$} < 1 \leq 1 \!-\! \epsilon$ for all $k$.
%
%
Hence, the result directly follows from (\ref{eq.ProbabilisticallySafeCorridor}) and the fact that for any safe configuration $\vect{p} \in \R^{\dimspace} \setminus \overline{\CR}_{\gmpdf\prl{\boldsymbol{\mean}_{\collspace}, \boldsymbol{\covmat}_{\collspace}, \boldsymbol{\weight}_{\collspace}}}\prl{\boldsymbol{\CL}_{\collspace}}$ the   probabilistically safe corridor $\safespace\prl{\vect{p}; \boldsymbol{\mean}_{\collspace}, \boldsymbol{\covmat}_{\collspace}, \boldsymbol{\CL}_{\collspace}}$ is bounded by tangent hyperplanes of confidence regions of individual Gaussians that strictly separates the point $\vect{p}$ from the Gaussian confidence ellipsoids. 
\end{proof}

 Note that the safe corridor $\safespace_{\collspace}\prl{\vect{p}}$ around a probabilistically unsafe configuration $\vect{p} \in \overline{\CR}_{\gmpdf\prl{\boldsymbol{\mean}_{\collspace}, \boldsymbol{\covmat}_{\collspace}, \boldsymbol{\weight}_{\collspace}}}\prl{\boldsymbol{\CL}_{\collspace}}$ can be empty for $\epsilon < 0$, especially for Gaussian mixture models with significant  overlap.
 Fortunately, many Gaussian mixture learning algorithms yield proper mixture models with minimal overlap.
Moreover, in order to resolve this issue, one can consider using a nonnegative $\epsilon$, which adaptively relaxes the safety constraints of $\safespace_{\collspace}\prl{\vect{p}}$ depending on the safety level of the configuration $\vect{p}$ and yields a nonempty relatively safe corridor $\safespace_{\collspace}\prl{\vect{p}}$. 
Thus, an optimal selection of $\epsilon$ is $\epsilon = 0$, which ensures nonempty safe corridors for all configurations (Proposition \ref{prop.NonemptySafeSpace}) and  exact probabilistically safe corridors for  probabilistically safe configurations (Proposition \ref{prop.RealSafeSpace}).



\subsection{Guided Steering via Safe Corridors}

\begin{figure}[t]
\centering
\vspace{2mm}
\begin{tabular}{@{}c@{\hspace{4mm}}c@{}}
\includegraphics[width=0.20\textwidth]{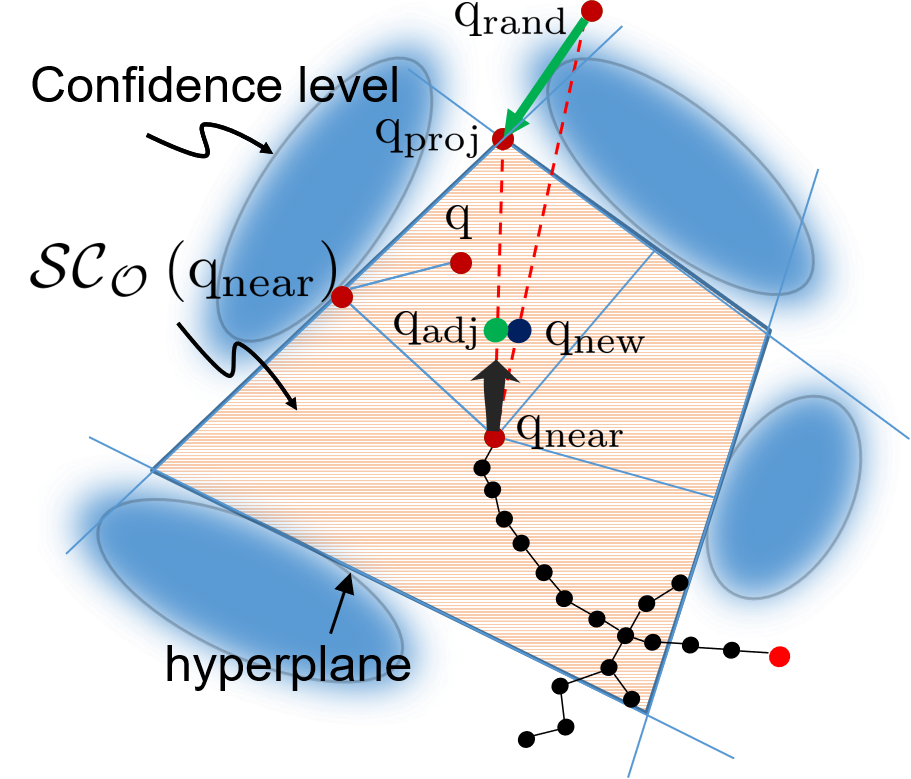} 
&
\includegraphics[width=0.23\textwidth]{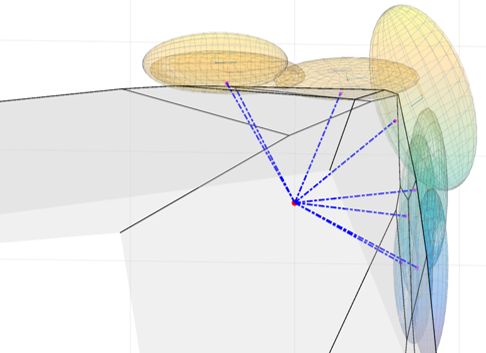}
\end{tabular}
\vspace{-1mm}
\caption{Local steering via probabilistically safe corridors. (left) Example tree extension using a probabilistically safe corridor in 2D space, (right) Probabilistically safe corridor in 3D space.}
\label{grd_fig_concept_guidance}
\vspace{-3mm}
\end{figure} 


We now describe a novel use of probabilistically safe corridors  for guided local steering of sampling-based planning, in particular, RRTs. In the original RRTs, a sample configuration $\vect{q}_{\mathrm{rand}}$ is randomly drawn in the configuration space, and then its nearest node $\vect{q}_{\mathrm{near}}$ in the tree is found based on a distance measure, which is set to be the standard Euclidean distance in this paper. Then, a new configuration $\vect{q}_{\mathrm{new}}$ is slightly extended from $\vect{q}_{\mathrm{near}}$ towards $\vect{q}_{\mathrm{rand}}$, say using the standard straight-line steering. 
If $\vect{q}_{\mathrm{new}}$ is collision-free, it is added to the tree as a new node, which is connected to the nearest node. If $\vect{q}_{\mathrm{new}}$ collides with an obstacle, then tree construction repeats with another $\vect{q}_{\mathrm{rand}}$.  

In this paper, we propose a new approach for tree expansion where  $\vect{q}_{\mathrm{new}}$ is adjusted to head towards collision-free space using probabilistically safe corridors $\safespace_{\collspace}$, as shown in Fig. \ref{grd_fig_concept_guidance}, by projecting $\vect{q}_{\mathrm{rand}}$ onto  $\safespace_{\collspace}\prl{\vect{q}_{\mathrm{near}}}$ as follows:
\begin{align}\label{eq.SafeGoal}
\vect{q}_{\mathrm{proj}} = \mtrcproj_{\safespace_{\collspace}\prl{\vect{q}_{\mathrm{near}}}}\prl{\vect{q}_{\mathrm{rand}}}
\end{align}
where $\mtrcproj_{A}\prl{\vect{x}} := \arg\min_{\vect{a} \in A} \norm{\vect{x} - \vect{a}}$ is the metric projection of a point $\vect{x} \in \R^{\dimspace}$ onto a closed convex set $A \subseteq \R^{\dimspace}$; that is to say, $\mtrcproj_{A}\prl{\vect{x}}$ returns the closest point of set $A$ to the input point $\vect{x}$.
Hence, the tree is extended towards $\vect{q}_{\mathrm{proj}}$ instead of $\vect{q}_{\mathrm{rand}}$, as shown in Fig. \ref{grd_fig_concept_guidance}.

\begin{proposition}
If a sampling-based motion planning algorithm is probabilistically complete for the standard straight-line steering, then the straight-line steering towards the projected goal onto probabilistically safe corridors, as described in (\ref{eq.SafeGoal}),  preserves its probabilistic completeness for  $\epsilon > 0$.
\end{proposition}
\begin{proof}
The result simply follows from Proposition~\ref{prop.NonemptySafeSpace} because the probabilistically safe corridor $\safespace_{\collspace}\prl{\vect{p}}$ of a configuration $\vect{p}\in \R^{\dimspace}$ strictly contains $\vect{p}$ in its interior  for $\epsilon > 0$ and the metric projection onto a probabilistically safe corridor locally behaves as the identity map. In other words, for $\epsilon > 0$, the straight-line steering toward the projected goal onto probabilistically safe corridors is locally equivalent to the standard unconstrained straight-line steering.    
\end{proof}

 \begin{algorithm}[t]
 \caption{Tree Extension in Configuration Space}\label{grd_config_extension_code}
 \begin{algorithmic}[1]
 \Require : $\boldsymbol{\mean}_{\collspace}$, $\boldsymbol{\covmat}_{\collspace}$ 
\State $\mathcal{T}.init(\vect{q}_{\mathrm{init}})$;
\While{Distance($\vect{q}_{\mathrm{goal}}$, $\vect{q}_{\mathrm{new}}$) $>$ $d_{min}$}
\State $\vect{q}_{\mathrm{rand}} \gets$ GetRandomSampling(), $iter = 0$; 
 \While{$iter < max\_iter$}
\State $\vect{q}_{\mathrm{near}} \gets$ GetNearestNeighbor($\mathcal{T}, \vect{q}_{\mathrm{rand}}$);
\State $\vect{q}_{\mathrm{proj}} \gets$ SteeringGuide($\boldsymbol{\mean}_{\collspace}$, $\boldsymbol{\covmat}_{\collspace}, \vect{q}_{\mathrm{near}},\vect{q}_{\mathrm{rand}}$);
\State $\vect{q}_{\mathrm{adj}} \gets$ StraightLineSteering($\vect{q}_{\mathrm{near}}, \vect{q}_{\mathrm{proj}}$,$\,\delta$);
\If {StraightLine($\vect{q}_{\mathrm{near}}$, $\vect{q}_{\mathrm{adj}}$) is Collision-Free}
\State $\mathcal{T}.addTree(\vect{q}_{\mathrm{adj}}$), $iter = iter + 1$;
 \Else
 \State \textbf{break};
 \EndIf 
 \EndWhile
 \EndWhile
 \end{algorithmic}
 \end{algorithm}

One computational challenge of our guided steering approach is that it requires to recompute the metric projection of $\vect{q}_{\mathrm{rand}}$ onto $\safespace_{\collspace}\prl{\vect{q}_{\mathrm{near}}}$ for each new selection of $\vect{q}_{\mathrm{rand}}$ and so $\vect{q}_{\mathrm{near}}$. 
Metric projection onto a convex polytope can be solved using any state-of-the-art quadratic optimization solver. For efficiency, we apply the active-set method for quadratic optimization, which is an iterative solver that ensures a feasible solution  and a decrement on the objective function at each iteration. This enables us to inherit some useful information from prior computation and stop its computation after some desired number of iterations. 
In order to reduce to computational cost, we keep $\vect{q}_{\mathrm{rand}}$ the same until a maximum number of iteration $max\_iter$ is reached. 
This enables us to warm-start the active set method with the  active constraints of the previous computation. If active constraints at the optimal solution are given, then a quadratic optimization problem with inequality constraints can be converted into a quadratic problem with equality constraints, which requires significantly less computational time to solve the optimization problem. For example, previous active constraints could be still active for slightly changed $\vect{q}_{\mathrm{near}}$ if the sample goal $\vect{q}_{\mathrm{rand}}$ is kept the same. Therefore, to increase computational efficiency, we always check first if the quadratic optimization is feasible with previously active hyperplane constraints of probabilistically safe corridors. 

\subsubsection{Tree Extension in the Configuration Space}

Algorithm 1 presents the pseudocode for the proposed tree extension methods in the configuration space. Here,  the nearest node $\vect{q}_{\mathrm{near}}$ of a random goal $\vect{q}_{\mathrm{rand}}$ in tree $\mathcal{T}$ is extended  by a new node $\vect{q}_{\mathrm{adj}}$  towards the projected goal $\vect{q}_{\mathrm{proj}}$  through the probabilistically safe corridor $\safespace_{\collspace}$ of $\vect{q}_{\mathrm{near}}$. If the random goal $\vect{q}_{\mathrm{rand}}$ satisfies the safety corridor constraints, then the tree is directly extended to the random goal, just like the standard straight-line extension method. In our implementation, we set the maximum number of iterations,  $max\_iter$ (Line 4), for using the same random goal $\vect{q}_{\mathrm{rand}}$ to be 3, and we select the maximum stepsize of the straight-line planner, $\delta$ (Line 7), manually depending on the desired accuracy level of collision checks. 



\subsubsection{Tree Extension in the Task Space}

For task space planning, we also use probabilistically safe corridors for guiding the end-effector of a manipulator  as described in Algorithm 2. 
Using forward kinematics, we define $\vect{X}_{\mathrm{rand}}$ to be the end-effector position of the random goal $\vect{q}_{\mathrm{rand}}$ and $\vect{X}_{\mathrm{near}}$ to be the end-effector position of the nearest node $\vect{q}_{\mathrm{near}}$ of $\vect{q}_{\mathrm{rand}}$ in tree $\mathcal{T}$. 
Here, our objective is to steer the end-effector position $\vect{X}_{\mathrm{near}}$ towards $\vect{X}_{\mathrm{rand}}$ via the projection $\vect{X}_{\mathrm{proj}}$ of $\vect{X}_{\mathrm{rand}}$ onto the $\safespace_{\collspace}\prl{\vect{X}_{\mathrm{near}}}$ along  the safe corridor $\safespace_{\collspace}\prl{\vect{X}_{\mathrm{near}}}$ in 3D space, as shown in Fig. \ref{grd_fig_concept_guidance}.
Accordingly, we select a steering step that is proportional with the stepsize of the standard straight-line steering of the end-effector as 
%
\begin{equation}
\Delta \vect{X}_{\mathrm{adj}} = \frac{\vect{X}_{\mathrm{proj}}-\vect{X}_{\mathrm{near}}}{||\vect{X}_{\mathrm{proj}}-\vect{X}_{\mathrm{{near}}}||}\cdot|| \vect{X}_{\mathrm{new}}-\vect{X}_{\mathrm{near}}|| , 
\end{equation}
%
%
and determine  the corresponding configuration as:
\begin{equation}
\vect{q}_{\textrm{adj}} = \vect{q}_{\textrm{near}} + J^{\dagger}(\vect{q}_{\textrm{near}}) \Delta \vect{X}_{\textrm{adj}}, 
\end{equation}
where $J^{\dagger}$ is the pseudoinverse of manipulator Jacobian $J$, satisfying  $J^{\dagger}  = J^{T} (J J^{T})^{-1}$. 
In  Fig. \ref{grd_fig:guidance_example}, we illustrate the guided steering of a manipulator using probabilistically safe corridors in task space: The new configuration (magenta), suggested by the standard straight line planner,  collides with obstacles, whereas the adjusted configuration (green), consistent with probabilistically safe corridors, moves in the tangent direction of obstacles.

\begin{algorithm}[t]
\caption{Tree Extension in Task Space}\label{grd_task_extension_code}
\begin{algorithmic}[1]
\Require : $\boldsymbol{\mean}_{\collspace}$, $\boldsymbol{\covmat}_{\collspace}$ 
\State $\mathcal{T}.init(\boldsymbol{e}_{\mathrm{init}}, \vect{q}_{\mathrm{init}})$;
\While{Distance($\vect{q}_{\mathrm{goal}}$, $\vect{q}_{\mathrm{new}}$) $>$ $d_{min}$}
\State $\vect{q}_{\mathrm{rand}} \gets$ GetRandomSampling();
\State $\vect{q}_{\mathrm{near}} \gets$ GetNearestNeighbor($\mathcal{T}, \vect{q}_{\mathrm{rand}}$);
\State $\vect{q}_{\mathrm{new}} \gets$ StraightLineSteering($\vect{q}_{\mathrm{near}}, \vect{q}_{\mathrm{rand}}$, $\delta$);
\State $\vect{X}_{\mathrm{rand}},\vect{X}_{\mathrm{near}}, \vect{X}_{\mathrm{new}} \gets$FwdKin($\vect{q}_{\mathrm{rand}},\vect{q}_{\mathrm{near}}, \vect{q}_{\mathrm{new}}$);
\State $\vect{X}_{\mathrm{proj}} \gets$SteeringGuide($\boldsymbol{\mean}_{\collspace},\boldsymbol{\covmat}_{\collspace},\vect{X}_{\mathrm{near}},\vect{X}_{\mathrm{rand}}$);
\State $\Delta \vect{X}_{\mathrm{\mathrm{adj}}} \gets$ $\frac{\vect{X}_{\mathrm{proj}}-\vect{X}_{\mathrm{near}}}{||\vect{X}_{\mathrm{proj}}-\vect{X}_{\mathrm{near}}||}\cdot|| \vect{X}_{\mathrm{new}}-\vect{X}_{\mathrm{near}}||$;
\State $\vect{q}_{\mathrm{adj}} \gets$ $\vect{q}_{\mathrm{near}} + J^{\dagger}(\vect{q}_{\mathrm{near}}) \Delta \vect{X}_{\mathrm{adj}}$ ;
\If {StraightLine($\vect{q}_{\mathrm{near}}$,$\vect{q}_{\mathrm{adj}}$) is Collision-Free}
\State $\mathcal{T}.addTree(q_{\mathrm{adj}}$);
\EndIf 
\EndWhile
\end{algorithmic}
\end{algorithm}

\begin{figure}[t]
\vspace{2mm}
\centering
\includegraphics[width=0.48\textwidth]{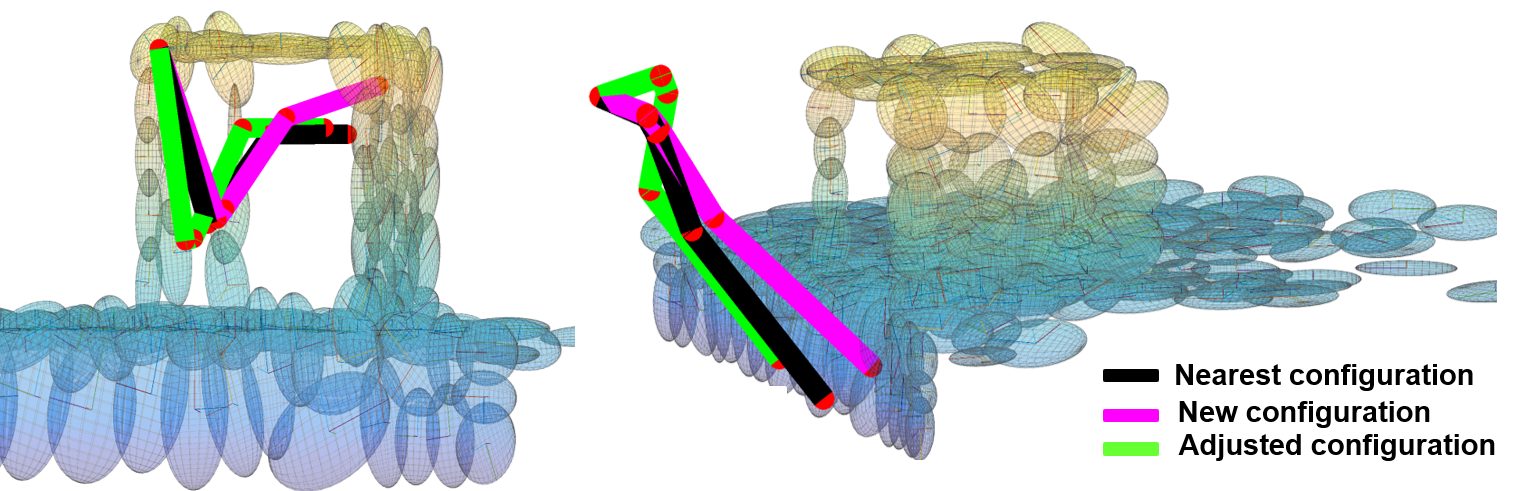}
\label{grd_fig:sim_condition}
\vspace{-3mm}
\caption{Examples of task-space steering of a robotic manipulator. Here, the new configuration (magenta), suggested by the straight line  planner from the nearest configuration (black), is adjusted to a better configuration (green) based on the associated probabilistically safe corridor. }\label{grd_fig:guidance_example}
\vspace{-3mm}
\end{figure}





\subsubsection{GMM-based Biased Sampling}
\label{grd:gmm_sampling}
In our experiments, we also compute the mixtures of Gaussian $\gmpdf\prl{\vect{x}, \boldsymbol{\mean}_{\freespace}, \boldsymbol{\covmat}_{\freespace}, \boldsymbol{\weight}_{\freespace}}$ for modeling the free space, which is used for biased sampling over the free space as described in \cite{huh2016learning}. 
For the settings where biased sampling is used, instead of uniform sampling in Line 3 in Algorithms 1 and 2, we randomly sample a configuration from the collision-free Gaussian mixture distribution $\gmpdf\prl{\vect{x}, \boldsymbol{\mean}_{\freespace}, \boldsymbol{\covmat}_{\freespace}, \boldsymbol{\weight}_{\freespace}}$. 
This sampling method increases the likelihood of a new sample being collision-free, and so can increase the computational efficiency of planning as discussed below.




\section{Results}

We evaluate SG-RRT in various environments using both a simulator and a real robot. We analyze the performance of SG-RRT by comparison with several existing RRT approaches. In addition, we demonstrate  SG-RRT on a real humanoid robot and provide results under real settings. All experiments are performed on a 2.7GHz PC, and all planners are implemented in Matlab.


\subsection{Learning Gaussian Mixture Models}

In all our experiments, we learn Gaussian mixture models offline by using the samples generated during the standard RRT planning (which was rich enough for accurate modeling, see Fig. \ref{grd_2d_link_result:rrt_cs}) and by manually selecting the kernel bandwidth for the Meanshift clustering so that the desired level of representation resolution is guaranteed.
In particular, we select the Gaussian kernel sizes for the Meanshift clustering as 10 degrees for 2DoF manipulator planning, 20 degrees for 7DoF manipulator planning, and $5$ cm for task space planning.
GMM learning takes 1.61 seconds for 191 clusters from 10,000 collision samples for 2DoF manipulator, 58.97 seconds for 1,096 clusters from 19,456 collision samples for 7DoF manipulator, and 3.64 seconds for 189 clusters from a 3D point cloud (including 18,413 data points) for task space planning.
For probabilistically safe corridors, we set the desired confidence level  $\kappa= 0.9$ and the safety tolerance $\epsilon=0.01$ for all cases. 
In future work, we plan to consider online GMM learning for adaptive motion planning in dynamic environments.

\subsection{2DoF Planar Manipulator}

\begin{figure}[t]
\vspace{2mm}
\centering
\subfigure[Workspace]
{\includegraphics[width=0.17\textwidth]{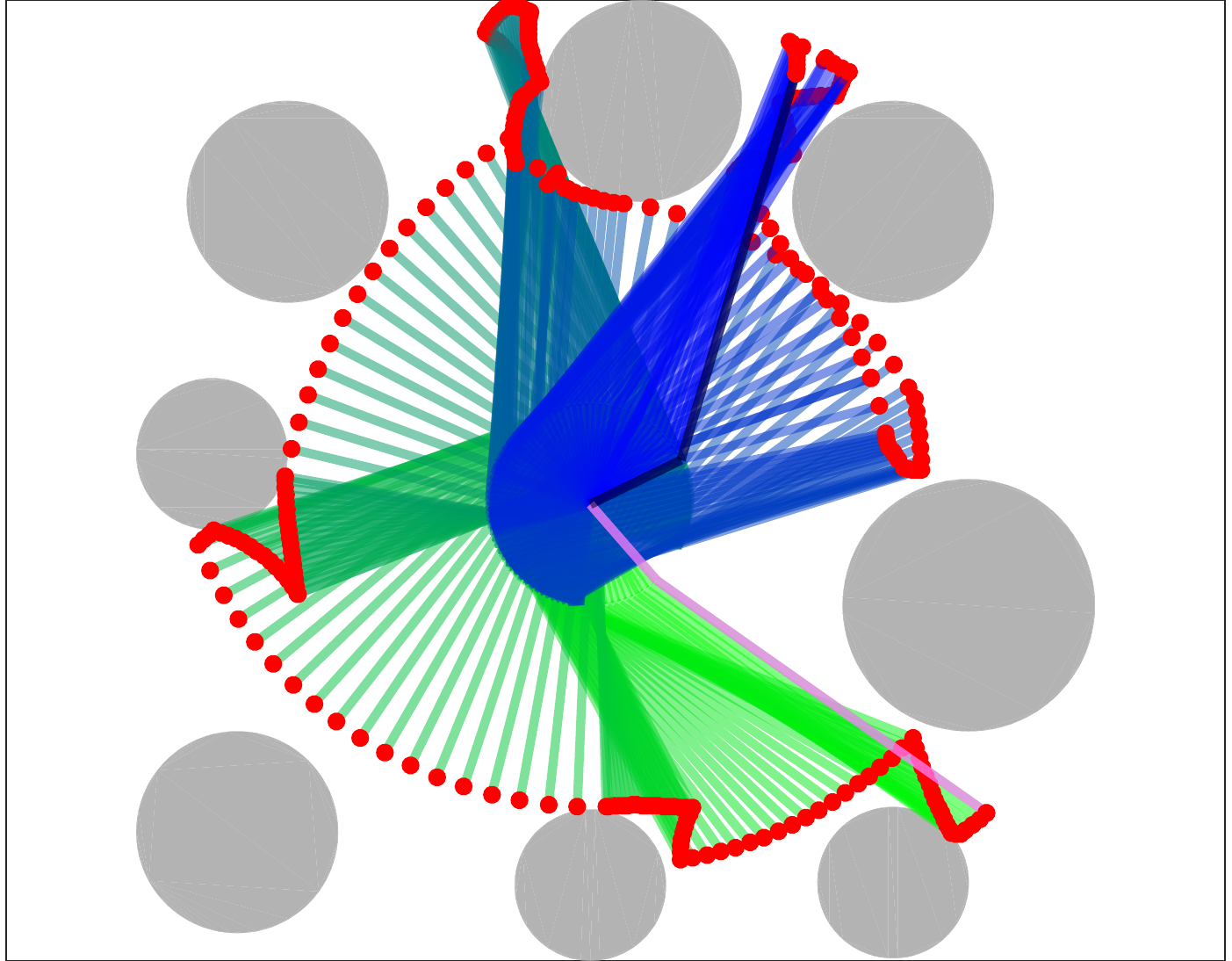}
\label{grd_2d_link_result:ws}
}
\subfigure[RRT]
{\includegraphics[width=0.135\textwidth]{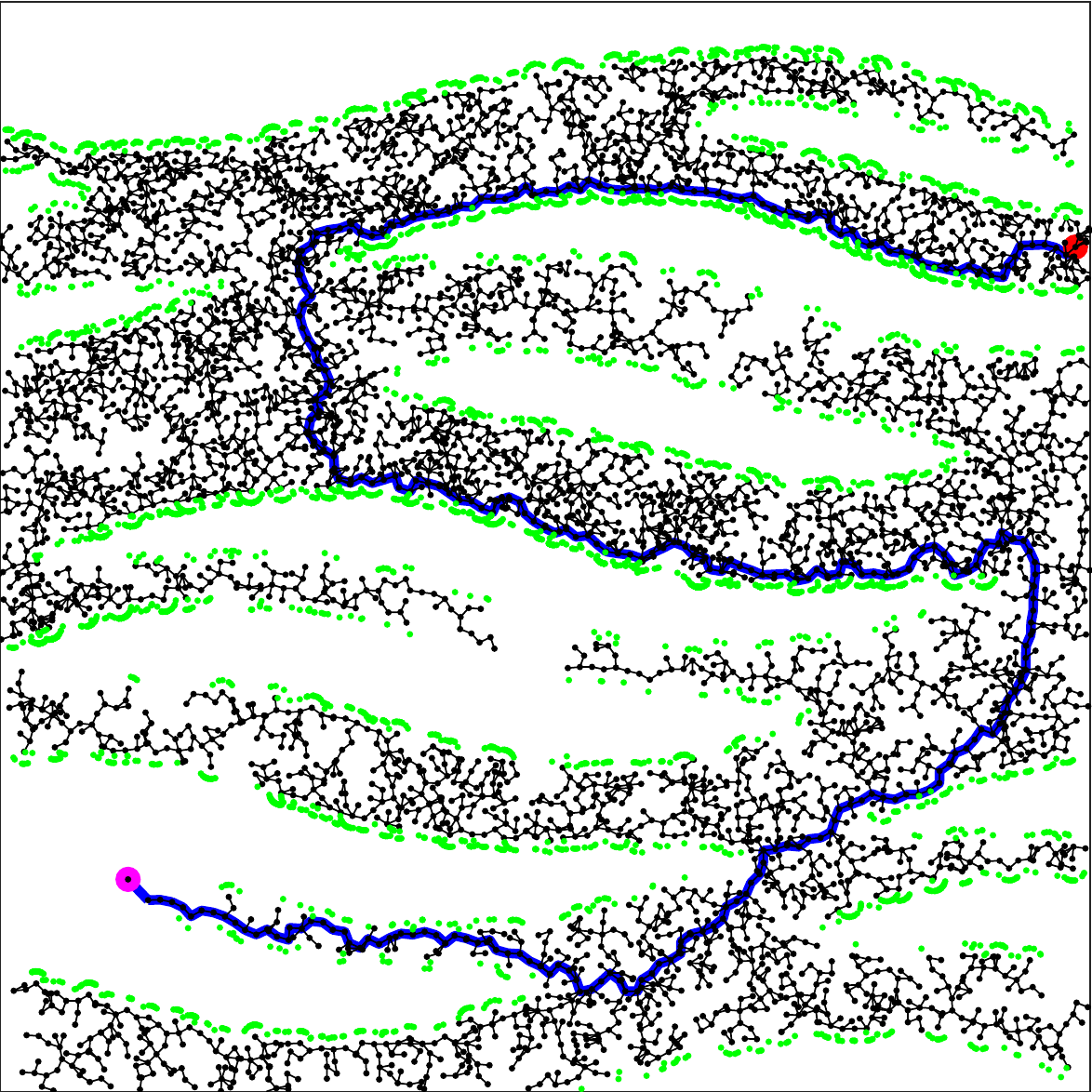}
\label{grd_2d_link_result:rrt_cs}
}
\subfigure[SG-RRT]
{\includegraphics[width=0.135\textwidth]{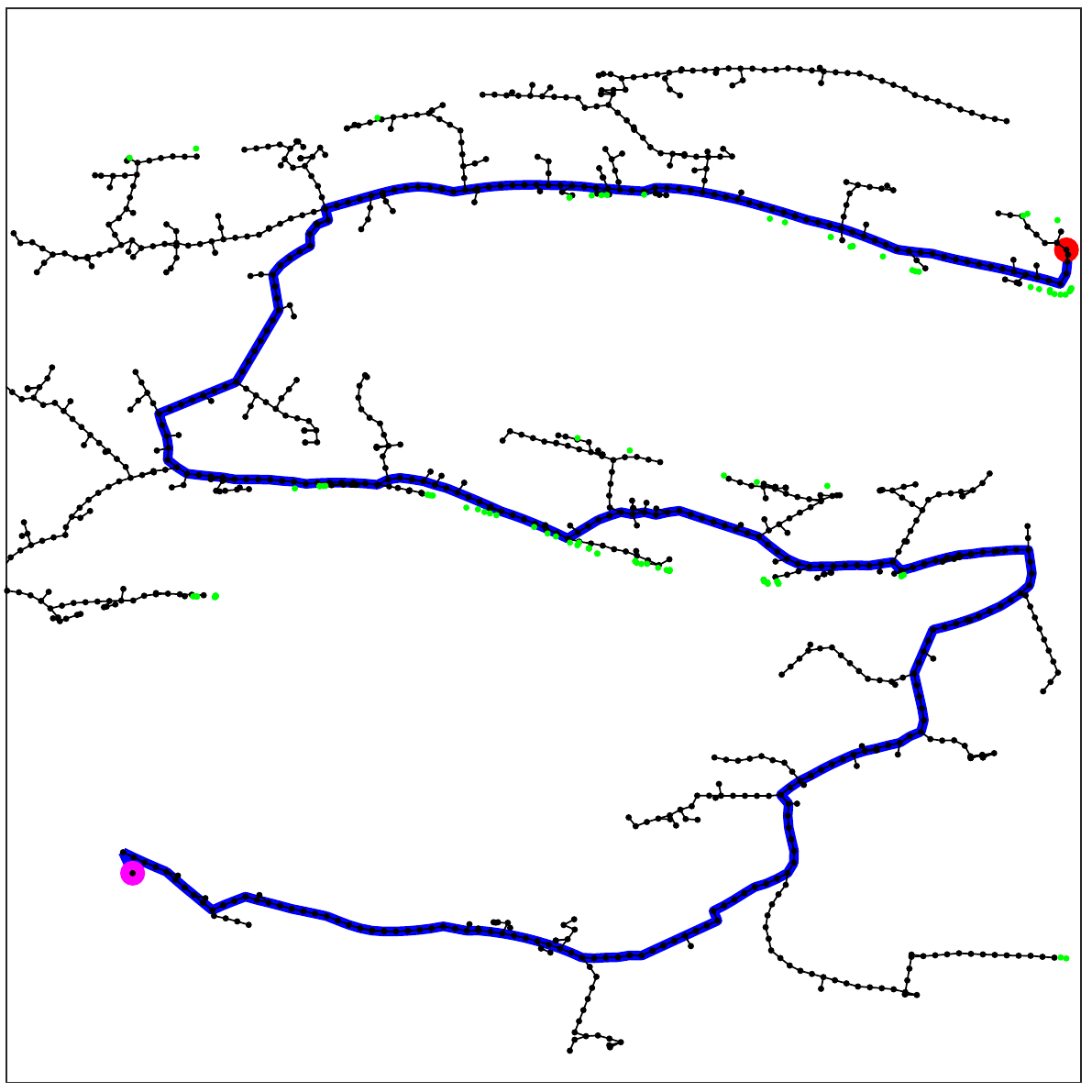}
\label{grd_2d_link_result:sgrrt_cs}
}
\subfigure[Execution Time]
{\includegraphics[width=0.47\textwidth]{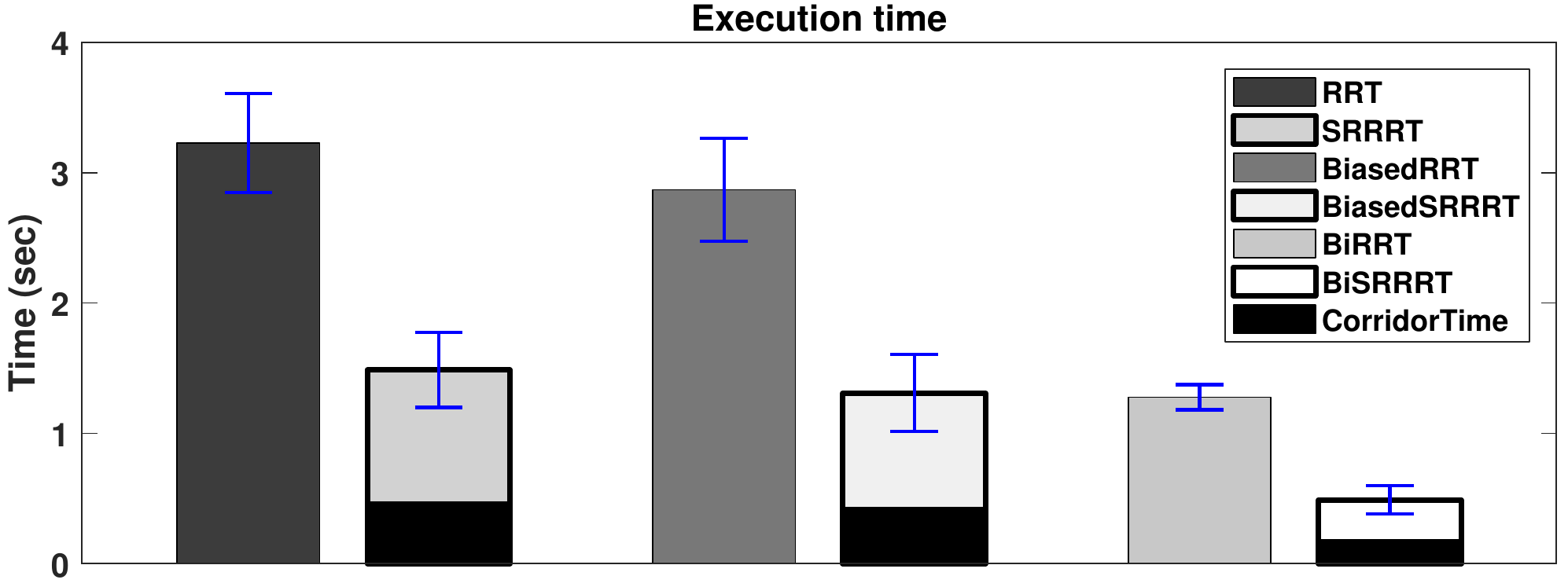}
\label{grd_2d_link_result:time}
}
\subfigure[Number of Collision Checks]
{\includegraphics[width=0.47\textwidth]{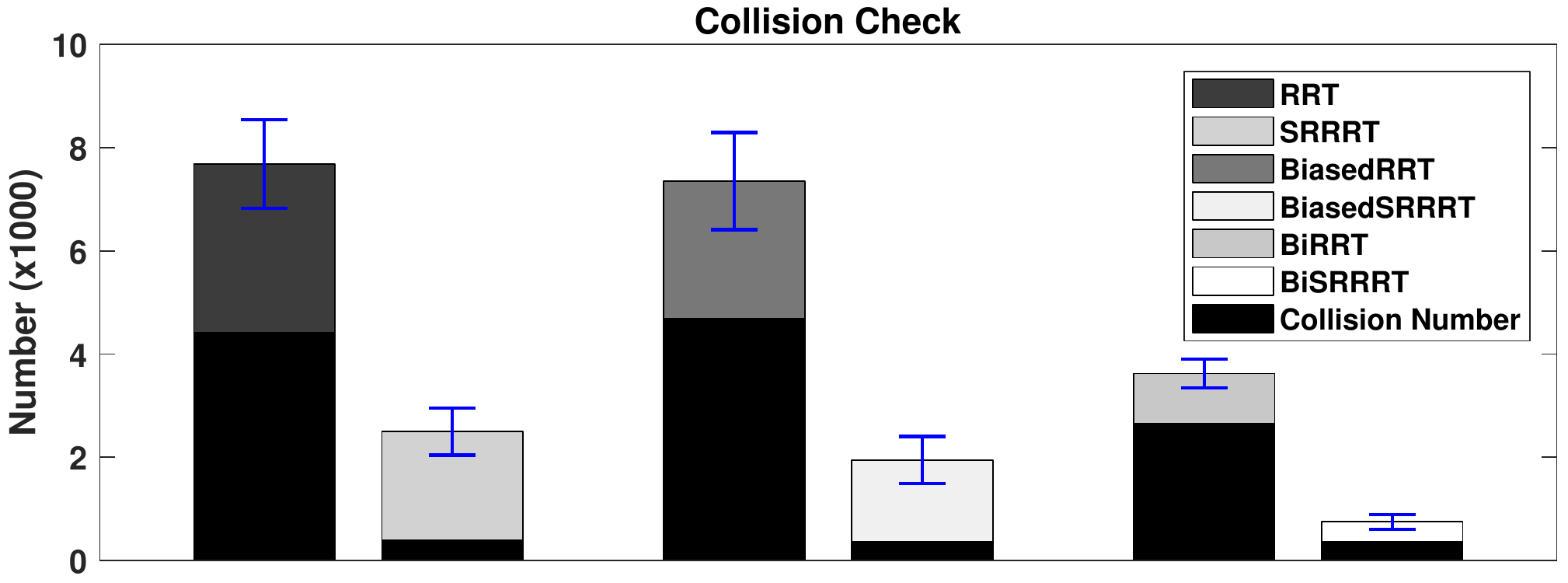}
\label{grd_2d_link_result:ccnum}
}
\vspace{-2mm}
\caption{RRT planning performance for a 2 DoF planar  manipulator
}
\label{grd_2d_link_result}
\vspace{-3mm}
\end{figure}  

\begin{figure}[b]
\vspace{-3mm}
\centering
\includegraphics[width=0.43\textwidth]{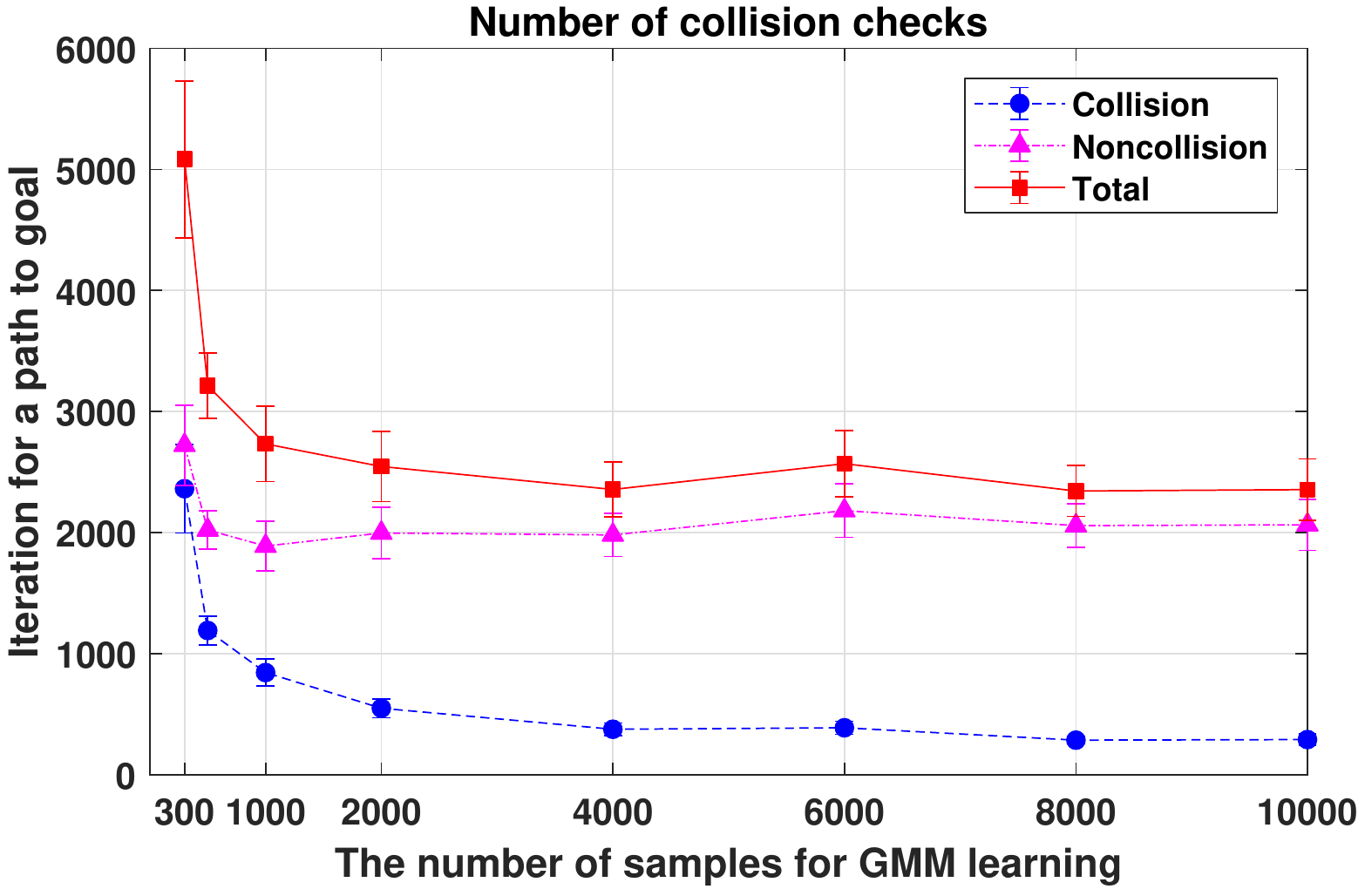} 
\vspace{-1mm}
\caption{Safety-guided RRT planning performance with respect to the number of collision samples used for GMM learning}
\label{fig:performance_sample_curve}
\end{figure}

\begin{figure}[h]
\vspace{2mm}
\centering
\begin{tabular}{cc}
  \includegraphics[width=0.20\textwidth]{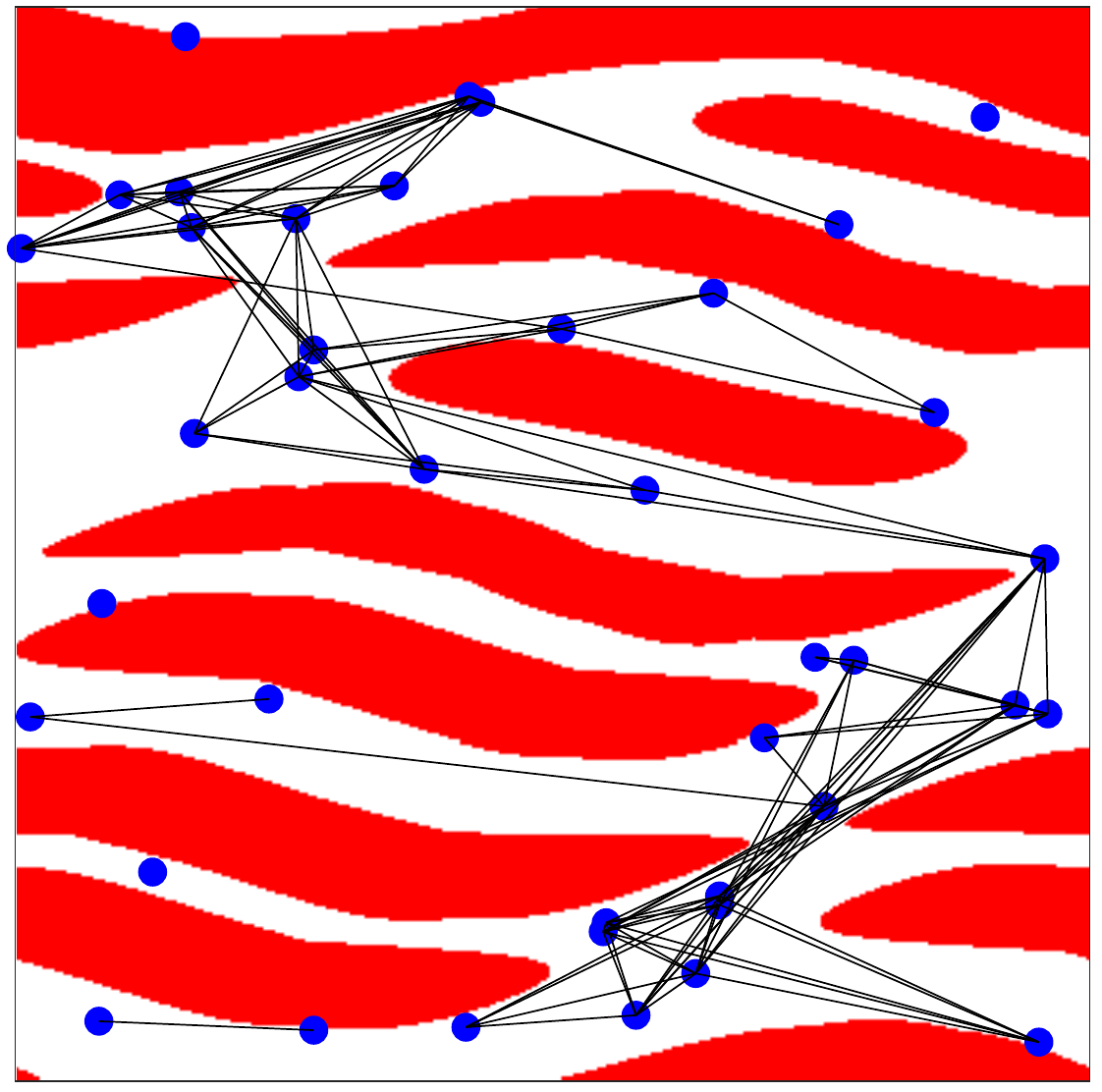}   &  \includegraphics[width=0.20\textwidth]{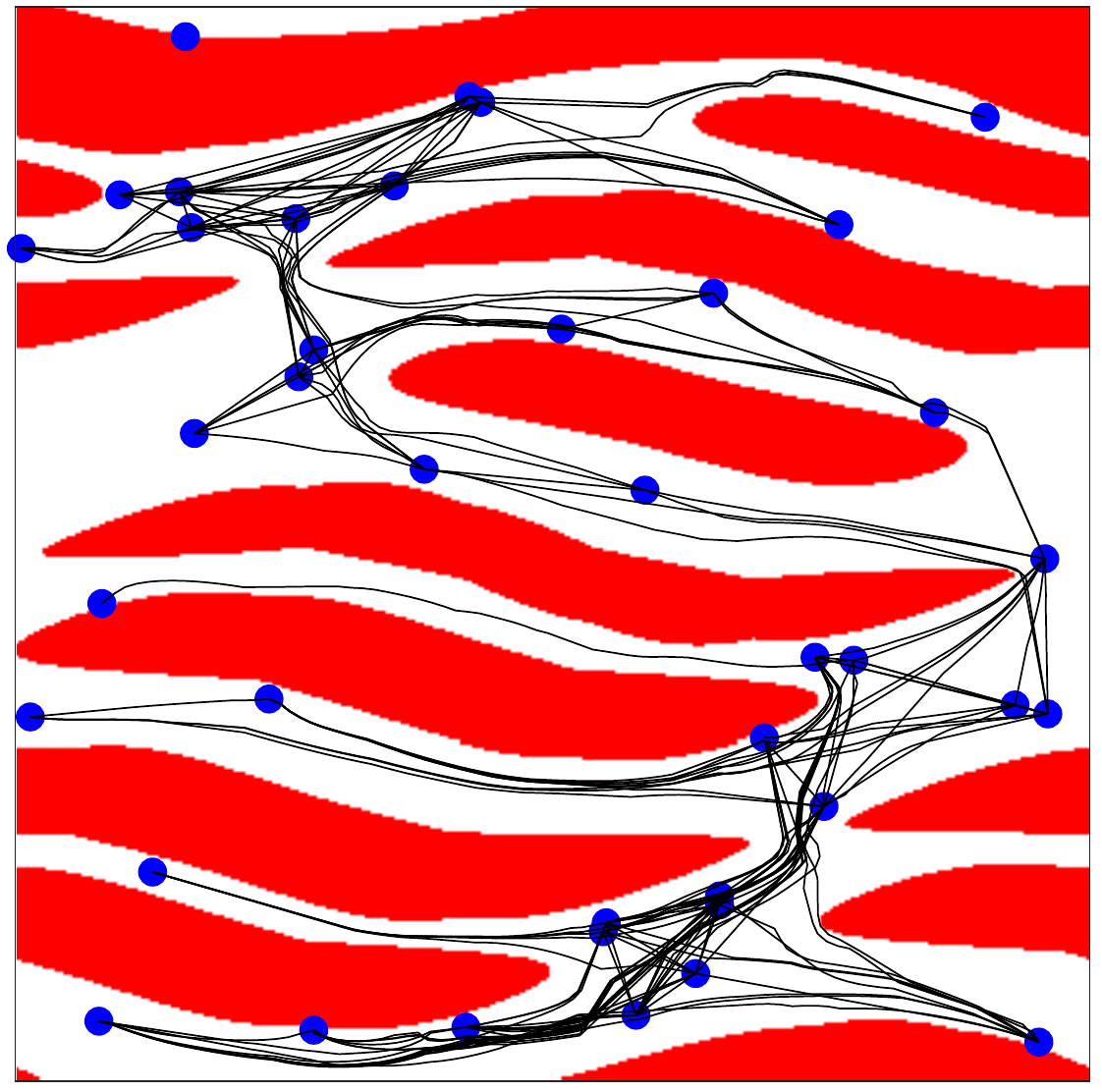}
\end{tabular}
\vspace{-2mm}
\caption{(left) PRM with the standard straight-line planner, (right) PRM with our safety guided local planner}
\label{fig:fig_prm_local_planner}
\vspace{-4mm}
\end{figure}

\begin{figure*}[!t]
\vspace{2mm}
\centering
\begin{tabular}{c}
\includegraphics[width=0.975\textwidth]{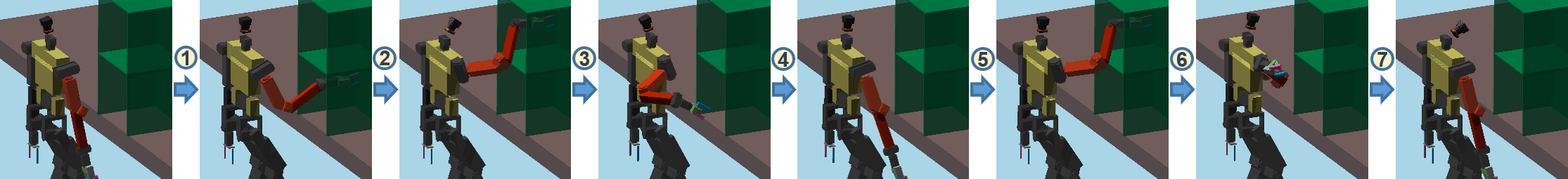}
\\[1mm]
\includegraphics[width=0.975\textwidth]{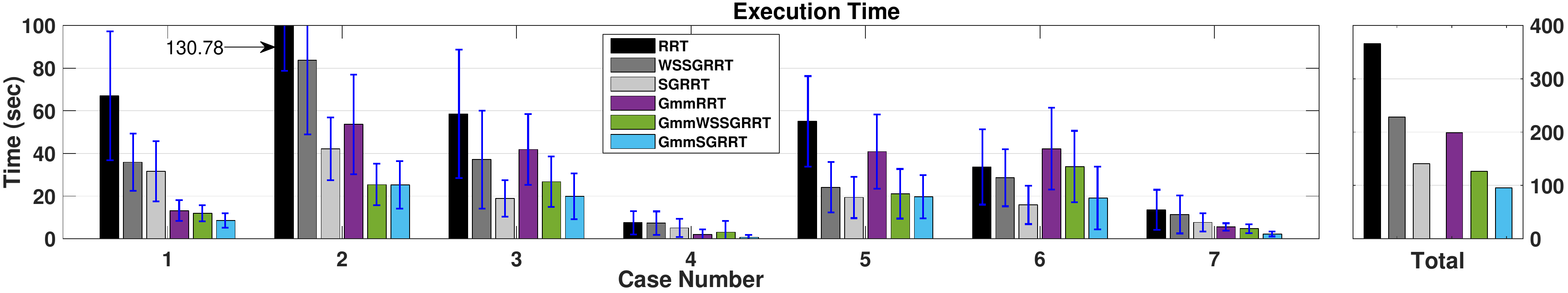} 
\\[0mm]
\includegraphics[width=0.975\textwidth]{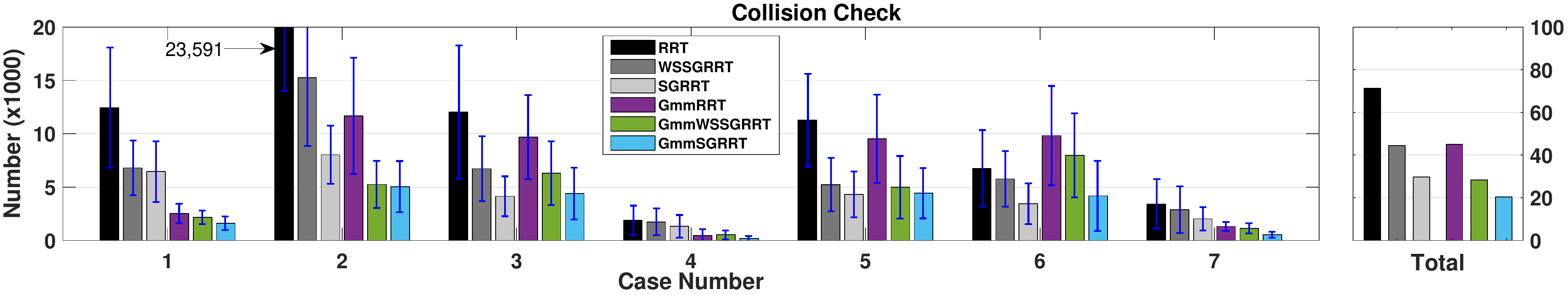}
\end{tabular}
\vspace{-2mm}
\caption{RRT planning performance for a 7DoF manipulator: (top)  Sequential planning tasks, (middle) Average execution time, (bottom) Average number of collision checks}
\label{grd_7d_link_result}
\vspace{-3mm}
 \end{figure*}  

For ease of visual presentation, we first consider motion planning of a 2DoF planar manipulator whose first link is 0.4 units long and second link is 1.6 units long as illustrated in Fig. \ref{grd_2d_link_result:ws}. 
In Fig. \ref{grd_2d_link_result}, we compare the computational performance of several variants of RRT planners (the standard RRT, the biased-RRT with $10\%$ goal bias, and the bidirectional RRT) with and without our proposed safety guided steering. 
Here, GMMs are learned offline along the collision space boundary (as shown in Fig. \ref{gmm_confidence_region}\,(d)) using  collision samples obtained during the standard RRT planning (green points in Fig. \ref{grd_2d_link_result:rrt_cs}) and they are used online for constructing probabilistically safe corridors. 
In our quantitative evaluation,  we consider the total execution time and the total number of collision checks as a performance measure, and we obtain the statistics (average and standard deviation) of these performance measures by running each planning algorithm for 50 times for 20 different start and goal pairs.
In overall, we observe that our safety guided steering increases computation performance significantly over the standard straight-line steering by dramatically reducing the required number of planning iterations (i.e., collision checks) to find a path between  any given start and goal pair, as shown in Fig. \ref{grd_2d_link_result:ccnum}. 
Because safety guided steering via probabilistically safe corridors minimizes collision risk by adaptively adjusting steering direction and stepsize.
As a result, our safety guided local planner yields steering action that are significantly less likely to be in collision; whereas the standard straight-line planner ends up being in collision with more than 50\% chance, as seen in Fig. \ref{grd_2d_link_result:ccnum}.
Finally, we find it useful to emphasize that the construction of and the projection onto a probabilistically safety corridor takes around 0.2 msec in average for each new sample (denoted by ``CorridorTime'' in Figure \ref{grd_2d_link_result}\,(d)), which is in the same order of magnitude as the computation cost of a collision check that takes around 0.3 msec.

{In Fig. \ref{fig:performance_sample_curve}, we demonstrate how the average number of RRT iterations (i.e., collision checks), required  for finding a path between any given start and goal pair, changes with the number of sample collision configurations (i.e., training data) used for Gaussian mixture learning. As expected, the performance of RRT planning with safety guided steering increases with the increasing size of training data as a result of increasing  accuracy of the Gaussian mixture model.


\begin{table}[b]
    \centering
    \vspace{-3mm}
    \caption{GMM and PRM Computation Times}
    \label{table:gmm_prm_construction}
    {\footnotesize
    \begin{tabular}{|@{\hspace{0.5mm}}c@{\hspace{0.5mm}}|@{\hspace{0.5mm}}c@{\hspace{0.5mm}}|@{\hspace{0.5mm}}c@{\hspace{0.5mm}}|@{\hspace{0.5mm}}c@{\hspace{0.5mm}}|@{\hspace{0.5mm}}c@{\hspace{0.5mm}}|@{\hspace{0.5mm}}c@{\hspace{0.5mm}}|@{\hspace{0.5mm}}c@{\hspace{0.5mm}}|@{\hspace{0.5mm}}c@{\hspace{0.5mm}}|}
    \hline
        \multicolumn{4}{|@{\hspace{0.5mm}}c@{\hspace{0.5mm}}|@{\hspace{0.5mm}}}{GMM Construction Time (sec)} & \multicolumn{4}{c|}{PRM Construction Time (sec)} \\  
        \hline
         Num.\,of  &Sampling  &GMM  &Total & Num.\,of  & PRM  & Collision &  Connected\\
         Samples  &Time  &Time & Time & Vertices & Time  &Checks & PRM\\
         \hline
         300               & 0.1665        & 0.0489       & 0.2154 & 100             & 2.4750          & 7,983             &  No \\
         500               & 0.2023        & 0.1220       & 0.3244 & 150             & 5.4377          & 18,361            &  No \\
         1,000             & 0.3855        & 0.2632       & 0.6487 & 200             & 10.3674          & 35,306            & No\\
         2,000             & 0.7640        & 0.4236       & 1.1876 & 250             & 16.0856          & 55,517            & No \\
         4,000             & 1.5159        & 0.8545       & 2.3704  &
         300             & 23.0590          & 81,274            & Yes \\
         6,000             & 2.2659        & 1.2205       & 3.4904 & 350             & 30.2114         & 107,841           & Yes \\
         8,000             & 2.8811        & 1.4230       & 4.3011 & 400             & 38.6380         & 134,851           & Yes \\
         10,000            & 3.6009        & 1.6100       & 5.2109 & 450             & 49.1138         & 171,122           & Yes \\ 
         \hline
    \end{tabular}
    }
\end{table}

 In Fig. \ref{fig:fig_prm_local_planner}, we present  an application of our safety guided steering to the probabilistic roadmap (PRM) planning of the 2DoF planar manipulator.
 As seen in Fig. \ref{fig:fig_prm_local_planner}, our safety guided steering noticeably increases the connectivity of a PRM as compared to the standard straight-line planner. 
 Here, two vertices of a PRM is said to be connected if safety guided steering can joining them in at most 100 steps.
 Finally, to briefly compare the computation cost of the learning phases of  the GMM and PRM methods, we provide in Table \ref{table:gmm_prm_construction}  the average computation time for the GMM and  PRM constructions for the 2DoF planar manipulator planning.
 As expected, for the same number of samples, GMM learning is around two orders of magnitude faster then the PRM construction because the connectivity test of PRMs is significantly computationally costly than the nearest neighbor search and the statistics computation of GMM.  
\subsection{7DoF Manipulator in 3D Space}

In order to validate the performance of SG-RRT quantitatively in high dimensional space, we compare it with traditional approaches with a 7DoF manipulator in 3D space using the Webots simulator of the Cyberbotics Ltd. company.
Fig. \ref{grd_7d_link_result}\,(top) shows the simulation scenario that is composed of seven sequential planning tasks. 
This scenario includes a difficult task, where the robot must remove its arm from the lower shelf and then insert it into the upper shelf. The simulation trials are repeated 50 times for accurate evaluation, and we use the average execution time and the number of collision checks as the evaluation criteria. 



 
For the comparison, we evaluate the standard RRT, safe-guided RRT (SG-RRT), and safe-guided RRT in the task space (WSSG-RRT). In addition, since we can apply GMM-based sampling as described in Section \ref{grd:gmm_sampling}, we also evaluate GMM-based RRT (Gmm-RRT), GMM-based safe-guided RRT (GmmSG-RRT), and GMM-based safe-guided RRT in the task space (GmmWSSG-RRT). Note that we apply a bidirectional method (RRT-Connect) \cite{kuffner2000rrt} in all approaches. The Gmm-RRT can be faster than the standard RRT, and the GmmSG-RRT is the fastest among all approaches. The WSSG-RRT and the GmmWSSG-RRT are faster than the RRT and Gmm-RRT. This demonstrates that the end-effector of the manipulator is effectively guided by the safe corridor in the high dimensional space, and it can reduce the computational time and the number of collision checks compared to traditional approaches.
We also observe in Fig. \ref{grd_7d_link_result} that SGRRT planning  is faster and requires less collision checks in configuration spaces than in task spaces, because probabilistically safe corridors are geometrically more informative when constructed in configuration spaces than in task spaces.   
Therefore, the tree extension with the safe corridor is significantly more efficient than the traditional methods.

\subsection{Physical Robot Experiments}

We demonstrate the performance of SG-RRT on a 7DoF manipulator (length: $85cm$) of an actual humanoid robot and an RGBD camera (ASUS Xtion Live Pro) with the scenario shown in Fig. \ref{grd_real_7d_link_result}\,(top). The robot is positioned $35cm$ from the shelf ($35cm \times 37 cm$) on the table. Figure \ref{grd_real_7d_link_result}  presents the comparison results of GmmSG-RRT and the standard RRT in terms of the execution time and the number of collision checks. Note that we apply a bidirectional method (RRT-Connect) and give 10$\%$ goal biased samples. Since the GmmSG-RRT adjusts a new node in the direction that avoids obstacles using probabilistically safe corridors and also utilizes biased sampling over collision-free space, the sample connectivity increases around narrow spaces, and tree expansion efficiently avoids obstacles. GmmSG-RRT is significantly efficient even when the robot needs to insert its arm onto the shelf. On the other hand, the computational time and the number of collision checks for the standard RRT planner dramatically increases in such complicated tasks.

\begin{figure}[h]
\vspace{2mm}
\centering
\begin{tabular}{c}
\includegraphics[width=0.46\textwidth]{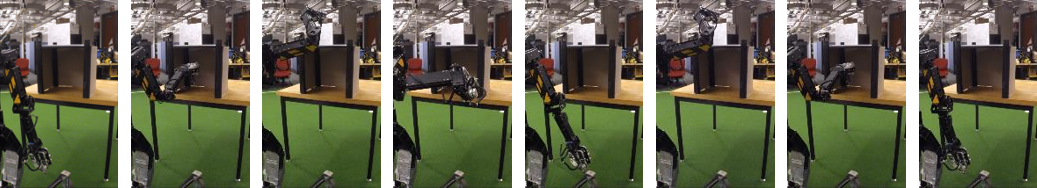}
\\
\includegraphics[width=0.46\textwidth]{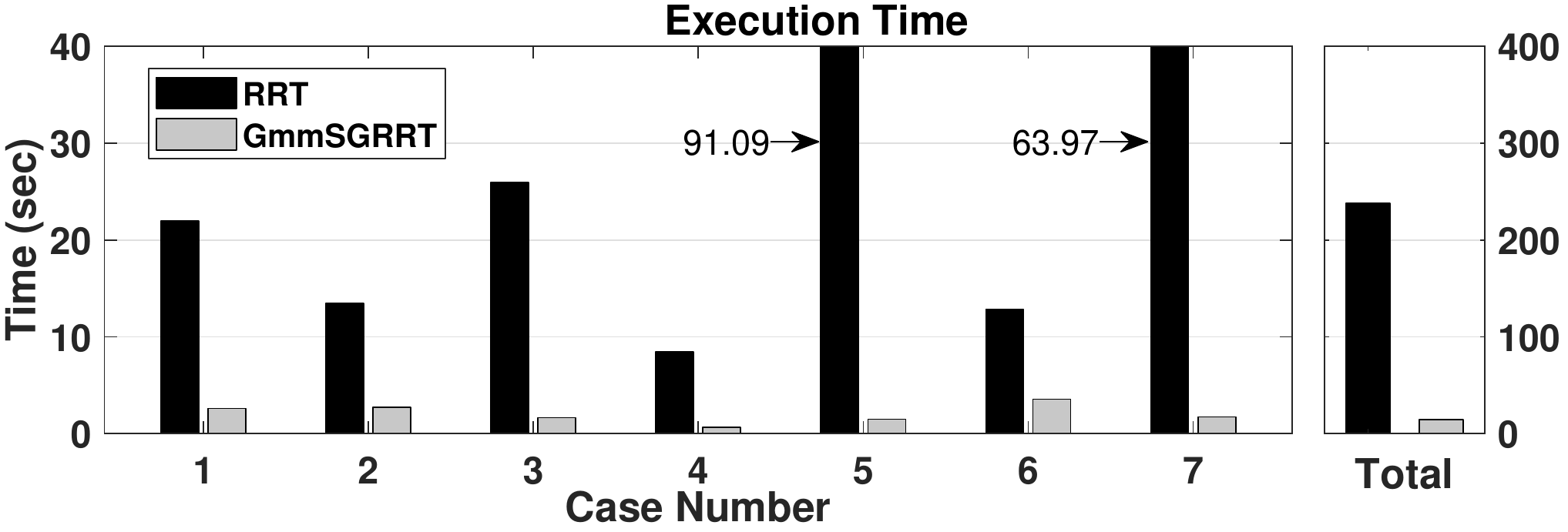}
\\
\includegraphics[width=0.46\textwidth]{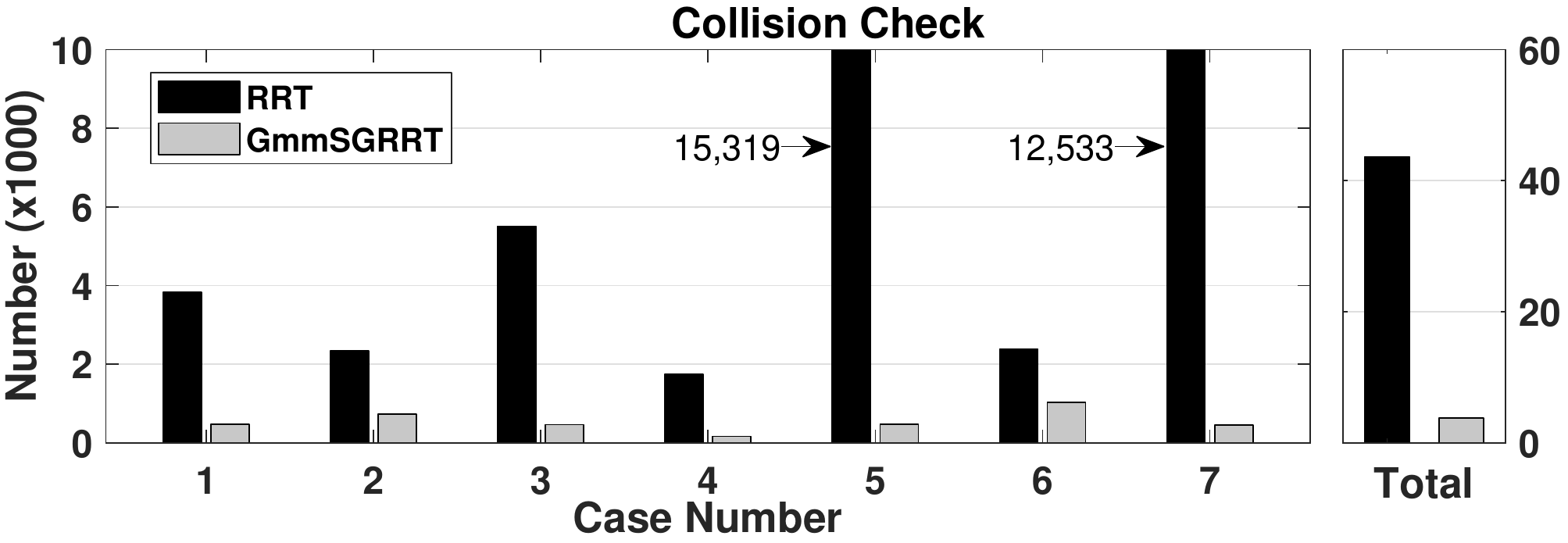}
\end{tabular}
\vspace{-1mm}
\caption{RRT planning performance with an actual physical robot: (top) Experiment with a physical robot, (middle) Average execution time, (bottom) Average number of collision checks
}
\label{grd_real_7d_link_result}
\vspace{-3mm}
\end{figure}  





\section{Discussion} 


In this paper, we present an effective local steering approach for sampling-based motion planning using probabilistically safe corridors of learned Gaussian mixture models of configuration spaces. 
We construct a probabilistically safe corridor around a configuration using tangent hyperplanes of confidence ellipsoids of Gaussian mixture models that are learned using collision history to approximate configuration space obstacles.
Accordingly, we propose a probabilistically safe local steering primitive that extends a random motion planning graph towards a sample goal using its projection onto the associated probabilistically safe corridor, which heuristically minimizes collision likelihood.
We observe that the proposed local steering approach improves the performance of sampling-based planning in challenging regions, especially narrow passages, by adjusting steering direction and stepsize.  
In our simulations and experiments with a real robot manipulator, we demonstrate that our proposed safety guided local  planner shows significant performance improvement over the standard straight-line planner for randomized motion  planning of 2DoF and 7DoF manipulators. In a future paper, we plan to extend our work using online GMM learning for uncertainty-aware adaptive planning.

{\small
\bibliographystyle{IEEEtran}
\bibliography{references.bib}
}

\end{document}